\documentclass{article}

    \usepackage[final]{neurips_2025}

\usepackage[utf8]{inputenc} 
\usepackage[T1]{fontenc}    
\usepackage{hyperref}       
\usepackage{url}            
\usepackage{booktabs}       
\usepackage{amsmath,amsfonts,amssymb,amsthm,mathtools,bm}       
\usepackage{nicefrac}       
\usepackage{microtype}      
\usepackage{xcolor}         
\usepackage{natbib}

\def\III{\text I}

\def\D{\mathcal D}
\def\F{\mathcal F}

\def\L{\mathcal L}

\def\N{\mathcal N}
\def\P{\mathcal P}

\def\R{\mathcal R}
\def\S{\mathcal S}

\def\EE{\mathbb E}
\def\NN{\mathbb N}
\def\PP{\mathbb P}
\def\RR{\mathbb R}

\def\ZZ{\mathbb Z}

\def\bbeta{\boldsymbol \beta}

\newtheorem{assumption}{Assumption}
\newtheorem{theorem}{Theorem}
\newtheorem{proposition}{Proposition}

\newtheorem{lemma}{Lemma}

\title{Statistical Guarantees for High-Dimensional Stochastic Gradient Descent}

\author{%
  Jiaqi Li \\
  Department of Statistics\\
  University of Chicago\\
  Chicago, IL 60637 \\
  \texttt{jqli@uchicago.edu} \\
  \And
  Zhipeng Lou \\
  Department of Mathematics\\ 
  University of California, San Diego\\
  La Jolla, CA 92093 \\
  \texttt{zlou@ucsd.edu} \\
  \AND
  Johannes Schmidt-Hieber \\
  Department of Applied Mathematics\\ 
  University of Twente\\
  Enschede, Netherlands \\
  \texttt{a.j.schmidt-hieber@utwente.nl} \\
  \And
  Wei Biao Wu \\
  Department of Statistics\\
  University of Chicago\\
  Chicago, IL 60637 \\
  \texttt{wbwu@uchicago.edu} \\
}

\begin{document}

\maketitle

\begin{abstract}
  Stochastic Gradient Descent (SGD) and its Ruppert–Polyak averaged variant (ASGD) lie at the heart of modern large-scale learning, yet their theoretical properties in high-dimensional settings are rarely understood. In this paper, we provide rigorous statistical guarantees for constant learning-rate SGD and ASGD in high-dimensional regimes. Our key innovation is to transfer powerful tools from high-dimensional time series to online learning. Specifically, by viewing SGD as a nonlinear autoregressive process and adapting existing coupling techniques, we prove the geometric-moment contraction of high-dimensional SGD for constant learning rates, thereby establishing asymptotic stationarity of the iterates. Building on this, we derive the $q$-th moment convergence of SGD and ASGD for any $q\ge2$ in general $\ell^s$-norms, and, in particular, the $\ell^{\infty}$-norm that is frequently adopted in high-dimensional sparse or structured models. Furthermore, we provide sharp high-probability concentration analysis which entails the probabilistic bound of high-dimensional ASGD. Beyond closing a critical gap in SGD theory, our proposed framework offers a novel toolkit for analyzing a broad class of high-dimensional learning algorithms.
\end{abstract}

\section{Introduction}\label{sec_intro}

Stochastic gradient descent (SGD) has been a cornerstone in large-scale machine learning since the seminal work by \citet{robbins_stochastic_1951}. It is especially efficient in high-dimensional and overparameterized settings where the number of unknown parameters can exceed the number of training samples \citep{arpit_closer_2017,zhang_understanding_2017, he_deep_2016}. SGD can also be combined with regularization techniques such as dropout to prevent overfitting in large networks \citep{krizhevsky_imagenet_2012,srivastava_dropout_2014}. Despite the vast amount of theoretical work on SGD, generalization bounds of SGD in high-dimensional regimes remain limited \citep{ 2023arXiv230111235G}. Considering a strongly convex objective function, we provide statistical guarantees for constant learning-rate SGD and its Ruppert–Polyak averaged variant (ASGD) \citep{ruppert_efficient_1988,polyak_acceleration_1992} in high-dimensional settings.

Specifically, we consider a general optimization problem
\begin{equation}
    \label{eq_goal_sgd}
    \bbeta^* \in \arg\min_{\bbeta\in\RR^d}G(\bbeta), \mbox{ where }
    \bbeta\mapsto G(\bbeta) :=\EE_{\bm{\xi}\sim\Pi} g(\bbeta,\bm{\xi}),
\end{equation}
$g(\cdot)$ is the noise-perturbed measurement of $G(\cdot)$, and $\bm{\xi}$ denotes a random element sampled from some unknown distribution $\Pi$. Given i.i.d.\ random samples $\bm{\xi}_1,\bm{\xi}_2,\ldots$ and some initialization $\bbeta_0\in\RR^d$, the $k$-th SGD iteration is
\begin{equation}
    \label{eq_sgd_recursion}
    \bbeta_k = \bbeta_{k-1} - \alpha\nabla g(\bbeta_{k-1},\bm{\xi}_k),\quad k=1,2,\ldots,
\end{equation}
for some constant learning rate $\alpha>0$, and $\nabla g(\bbeta,\bm{\xi})=\nabla_{\bbeta} g(\bbeta,\bm{\xi})$ the stochastic gradient with respect to $\bbeta$. For $k\ge1$, the ASGD variant is defined by
\begin{equation}
    \label{eq_asgd}
    \bar\bbeta_k = \frac{1}{k}\sum_{i=1}^k\bbeta_i.
\end{equation}
We are interested in the high-dimensional setting where the parameter dimension $d$ can be very large. Here, a notable divide between empirical success and theoretical understanding is that practitioners often employ a large constant learning rate $\alpha$ in~\eqref{eq_sgd_recursion} to accelerate convergence in high-dimensional problems \citep{wu_how_2018,cohen_gradient_2020,cai_large_2024}. However, such choices can induce pronounced non-stationarity in the SGD iterates $\{\bbeta_k\}_{k\in\NN}$ 
which will not converge to a point but oscillates around the mean of a stationary distribution. In other words, $\bbeta_k$ is non-stationary but asymptotically stationary, which converges only in distribution as $k\rightarrow\infty$, while the mean of this distribution differs from the exact minimizer $\bbeta^*$ due to the non-diminishing bias of order $O(\alpha)$ \citep{dieuleveut_bridging_2020,merad_convergence_2023}. Classical theory mostly relies on decaying learning rates \citep{zhang_statistical_2004,nemirovski_robust_2009,jentzen_lower_2020,shi_learning_2023}. To address the non-stationarity issue, we apply powerful tools from nonlinear time series analysis \citep{wu_limit_2004} to online learning, particularly by adapting the coupling techniques to show the geometric-moment contraction of SGD for constant learning rates. Specifically, for any two SGD sequences $\{\bbeta_k\}_{k\in\NN}$ and $\{\bbeta_k'\}_{k\in\NN}$ that share the same random samples but have different initial vectors $\bbeta_0$ and $\bbeta_0'$, we show in Theorem~\ref{thm_gmc} that for all sufficiently small constant learning rates $\alpha$, the initialization is forgotten exponentially fast in the sense that
\begin{equation}
    \label{eq_gmc_intro}
    (\EE|\bbeta_k-\bbeta_k'|_s^q)^{1/q} \le r_{\alpha,s,q}^k|\bbeta_0-\bbeta_0'|_s \ \  \mbox{ holds for all } k \in \NN,
\end{equation}
for contraction speed $0 \le r_{\alpha,s,q} < 1$, and $|\cdot|_s$ the $\ell^s$-norm, that is,
\begin{equation*}
  \big|(v_1, \ldots, v_d)^{\top}\big|_s=\Big(\sum_{i=1}^d|v_i|^s\Big)^{1/s}, \,\, s \ge 1.
\end{equation*} 
This asserts the existence of a limiting stationary distribution of $\bbeta_k$ as $k\rightarrow\infty$, thereby facilitating a systematic convergence theory of SGD even in nonlinear, overparameterized models. 

Building on this new framework, we provide non-asymptotic bounds for higher-order moments of the SGD error in general $\ell^s$-norms for any finite $s\ge2$ beyond the usual $\ell^2$-norm, extendable to max-norm $\ell^{\infty}$ by choosing $s\approx\log(d)$. Notably, the $\ell^{\infty}$-norm is frequently adopted in high‐dimensional sparse or structured estimation \citep{wainwright_high-dimensional_2019}. See for instance, the max-norm convergence of the Lasso and Dantzig selector \citep{lounici_sup-norm_2008}; the pivotal method for sup‐norm bounds of the square‐root Lasso \citep{belloni_square-root_2011}; and the max-norm error control for confidence intervals in high-dimensional regression problems \citep{javanmard_confidence_2013}. In stochastic‐approximation (SA), \citet{wainwright_stochastic_2019} derived $\ell^{\infty}$-norm bounds for Q-learning with decaying learning rates; \citet{chen_concentration_2024} derived maximal concentration bounds for SA under arbitrary norms with decaying learning rates and with contraction as an assumption; \citet{agarwal_stochastic_2012} considered high-dimensional SA for strongly convex objectives with a sparse optimum, but using decaying learning rates and restricting the tails of stochastic gradients to be sub-Gaussian. To date, all the existing results are restricted to low‐dimensional settings or decaying learning rates and do not carry over to overparameterized models with constant learning rates. To address this gap, we derive a sharp \textit{high-dimensional moment inequality} (see Lemma~\ref{lemma_rio_Ls_moment_ineq}) valid for a broad class of learning problems, delivering explicit non-asymptotic bounds of $\EE|\bbeta_k-\bbeta^*|_s^q$ and its ASGD variant for any $q,s\ge2$ with mild conditions, together with matching complexity guarantees, i.e., given some target error $\varepsilon>0$ (see Proposition~\ref{prop_complexity}), the required number of iterations $k$ such that 
\begin{equation*}
    \EE|\bar\bbeta_k-\bbeta^*|_s^q \le \varepsilon.
\end{equation*}
Although moment bounds capture average‐case performance, a single execution of (A)SGD in practice demands high-probability guarantees \citep{valiant_theory_1984,vapnik_nature_2000,bach_non-strongly-convex_2013,durmus_tight_2021,zhong_probabilistic_2024}. Recent advances include a generic high-probability framework for both convex and nonconvex SGD with sub-Gaussian gradient noises \citep{liu_high_2023}, high-probability rates for clipped-SGD with heavy-tailed noises \citep{nguyen_improved_2023}, and high-probability guarantees for nonconvex stochastic approximation via robust gradient clipping \citep{li_high_2022}. However, these established high-probability bounds focus again on decaying learning rates and low dimension. Moreover, early work primarily addressed light‐tailed noises where the gradients are bounded or have exponential‐type moments \citep{nemirovski_robust_2009,rakhlin_making_2012,ghadimi_stochastic_2013,cardot_online_2017,harvey_tight_2019,mou_linear_2020,chen_recursive_2023}. For the cases that only admit a polynomial tail with finite $q$-th moment, \citet{lou_beyond_2022} were the first to derive a Nagaev–type inequality \citep{nagaev_large_1979} for low-dimensional ASGD. The rate was shown to be optimal but their bound heavily relies on the linearity of gradients and is only suitable for decaying learning rates. 
By leveraging a dependency‐adjusted \textit{functional dependence measure} in high‐dimensional time series \citep{zhang_gaussian_2017}, we derive a high‐probability concentration bounds for high‐dimensional ASGD with constant learning rates. Given a tolerance level $\delta\in(0,1)$ and a target error $\varepsilon>0$, we provide bounds for the required number of iterations $k$ to guarantee that
\begin{equation*}
    \PP\big(|\bar\bbeta_k-\bbeta^*|_s\le\varepsilon\big) \ge 1-\delta.
\end{equation*}
This tail-decay result (see Eq.~\eqref{eq_prob_bound}) is proved via a new Fuk-Nagaev-type inequality (see Theorem~\ref{Theorem_Fuk_Nagaev}) and complements our moment and complexity characterizations of large-step stochastic optimization.

\subsection{Our Contributions}
This paper contributes to theoretical advancements for understanding constant learning-rate SGD and its averaged variant (ASGD) in the challenging high-dimensional regime. Our main technical innovations and results include:

\textbf{(1) Handling Constant Learning Rates in High Dimensions.} In practice, large-scale machine learning models commonly deploy fixed, large learning rates to speed up optimization in high-dimensional settings. To address this, we introduce novel coupling techniques inspired by high-dimensional nonlinear time series and establish the asymptotic stationarity of the SGD iterates with arbitrary initialization \textit{(Section~\ref{sec_gmc})}.

\textbf{(2) Generalized Moment Convergence in $\ell^s$- and $\ell^{\infty}$-Norms.} By deriving a sharp high-dimensional moment inequality, we establish explicit, non-asymptotic $q$-th moment bounds for arbitrary $\ell^s$-norms of (A)SGD iterates for any $q\ge2$ and even integers $s$, generalizing previous theory primarily focusing on mean squared error (MSE) convergence with $q=s=2$. Our results extend naturally to the max-norm case (i.e., $\ell^{\infty}$) by selecting $s \approx \log(d)$, that is essential for modern sparse and structured estimation in high-dimensional data \textit{(Section~\ref{sec_moment})}.
    
\textbf{(3) High-Probability Tail Bounds.} While average-case (moment) bounds are informative, single runs require tail guarantees. We derive the first high-probability concentration bounds for ASGD in high-dimensional settings with constant learning rates. By developing a tight Fuk-Nagaev-type inequality using the coupling techniques in nonlinear time series, we control the algorithmic complexity required to achieve targeted accuracy with high confidence \textit{(Section~\ref{sec_concentration})}.

\subsection{Related Works} 

\textit{Stochastic Gradient Descent and its Variants.} The SGD algorithm can be traced back to \citet{robbins_stochastic_1951,kiefer_stochastic_1952}. Popular SGD variants include Nesterov’s accelerated gradient \citep{nesterov_method_1983}, AdaGrad \citep{duchi_adaptive_2011}, AdaDelta \citep{zeiler_adadelta_2012}, Adam \citep{kingma_adam_2014}, AMSGrad \citep{reddi_convergence_2018}, AdamW \citep{loshchilov_decoupled_2018}, SAG \citep{schmidt_minimizing_2017}, SVRG \citep{johnson_accelerating_2013}, SARAH \citep{nguyen_sarah_2017}, SPIDER \citep{fang_spider_2018} and Katyusha \citep{allen-zhu_katyusha_2017}. The theoretical foundations of SGD under decaying learning rates were established in the early studies by \citep{blum_approximation_1954,dvoretzky_stochastic_1956,sacks_asymptotic_1958}, with stronger almost-sure guarantees by \citet{fabian_asymptotic_1968,robbins_convergence_1971,ljung_analysis_1977,lai_stochastic_2003,wang_stochastic_2010}. Existing works for smooth, strongly‐convex objectives with decaying step sizes include \citet{ruppert_efficient_1988,polyak_acceleration_1992,nemirovski_robust_2009,bach_non-strongly-convex_2013,rakhlin_making_2012,2020_Neurips_SGD_Martikopolous} among others. Despite the rich literature on SGD, the theoretical understanding in high-dimensional settings remains limited. Exceptions are  \citet{paquette_sgd_2021,paquette_implicit_2022} who study high-dimensional SGD for the least-squares loss.

\textit{Constant Learning Rate.} 
In high‐dimensional scenarios, constant learning rates prevail due to simpler tuning procedures and faster convergence \citep{wang_large_2022}. More recent theoretical and empirical studies of large‐step SGD include \citet{wu_how_2018,cohen_gradient_2020} and the very recent \citet{cai_large_2024}, which formalize the resurgence of constant‐step methods in modern machine learning. A useful way to analyze constant–step SGD is to treat its iterates as a time‐homogeneous Markov chain \citep{pflug_stochastic_1986}, which makes it possible to characterize its long‐run behavior and stationary law. However, previous works only derived convergence in Wasserstein distance \citep{dieuleveut_bridging_2020,merad_convergence_2023}. Such convergence in probability measures can hardly provide refined (non)-asymptotics such as higher-moment convergence and concentration inequalities, and seems nontrivial to extend to high-dimensional regimes.

\textit{High-Dimensional Nonlinear Time Series.}  
An alternative approach for constant learning-rate SGD is to view it as an iterated random function \citep{dubins_invariant_1966,barnsley_iterated_1985,diaconis_iterated_1999,diaconis_random_2000}, or a nonlinear autoregressive (AR) process. This interpretation facilitates the theory of online learning with non-stationarity and complex dependency structures; see, for example, the recent work by \citet{li_asymptotics_2024} on SGD with dropout regularization building on the GMC framework \citep{wu_limit_2004}. To extend this systematic theory to high-dimensional settings, we adapt the coupling techniques in time series \citep{wu_nonlinear_2005,wu_strong_2007,wu_recursive_2009,wu_asymptotic_2011,xiao_covariance_2012,berkes_komlosmajortusnady_2014,wu_performance_2016,karmakar_optimal_2020}, especially the ones for high-dimensional regimes \citep{zhang_gaussian_2017,zhang_convergence_2021,li_ell2_2024} to online learning algorithms.

\subsection{Notation} 
Denote column vectors in $\RR^d$ by lowercase bold letters $\bm{x}=(x_1, \ldots, x_d)^{\top}$ and the $\ell^s$-norm of $\bm{x}$ by $|\bm{x}|_s=(\sum_{i=1}^d|x_i|^s)^{1/s}$, $s\ge1$. Write $\bm{x}^{\odot s}=(x_1^s,\ldots,x_d^s)^{\top}$. The expectation and covariance of random vectors are respectively denoted by $\EE[\cdot]$ and $\mathrm{Cov}(\cdot)$. For $q>0$ and a random variable $X$, we write $X\in\L^q$ iff $\lVert X\rVert_q=[\EE(|X|^q)]^{1/q}<\infty$. We denote matrices by uppercase letters. Given matrices $A$ and $B$ of compatible dimension, their matrix product is denoted by juxtaposition. Write $A^{\top}$ for the transpose of $A$ and $I_d$ for $d \times d$ identity matrix. For two positive number sequences $(a_n)$ and $(b_n)$, we say $a_n=O(b_n)$ (resp. $a_n\asymp b_n$) if there exists $c>0$ such that $a_n/b_n\le c$ (resp. $1/c\le a_n/b_n\le c$) for all large $n$. 
Let $(x_n)$ and $(y_n)$ be two sequences of random variables. Write $x_n=O_{\PP}(y_n)$ if for $\forall \epsilon>0$, there exists $c>0$ such that $\PP(|x_n/y_n|\le c)>1-\epsilon$ for all large $n$.

\begin{table}[htb!]
    \centering
    \begin{tabular}{cccc}
    \hline\hline
        Notation & Definition & Reference & Index Range\\
        \hline \\[-2ex]
         $\bbeta^*$ & minimizer of the loss function $G(\bbeta)$ & Eq.~\eqref{eq_goal_sgd} & / \\[0.5ex]
         $\bbeta_k$ & SGD iterates  & Eq.~\eqref{eq_sgd_recursion} & $k\in\NN$ \\[0.5ex]
         $\bbeta_k^{\circ}$ & stationary SGD iterates & Thm.~\eqref{thm_sgd_moment} & $k\in\ZZ$\\[0.5ex]     
         $\bar\bbeta_k$ & ASGD iterates & Eq.~\eqref{eq_asgd} & $k\in\NN$\\[0.5ex]
         $\bar\bbeta_k^{\circ}$ & stationary ASGD iterates & Eq.~\eqref{eq_asgd_stationary} & $k\in\ZZ$\\[0.5ex]     
         \hline\hline
    \end{tabular}
    \caption{List of the sequences defined in the paper.}
    \label{table_seq_notation}
\end{table}

\section{Convergence of SGD to a Stationary Distribution}\label{sec_gmc}

In this section, we establish the GMC property of high-dimensional SGD with constant learning rates. Our technique is to construct a smooth surrogate for the non-differentiable $\ell^\infty$-norm via the $\ell^s$-norm, so that standard gradient‐based tools become available. We defer the technical details to Section~\ref{subsec_ls_linf}. Furthermore, we provide a novel high-dimensional moment inequality (see Section~\ref{subsec_hd_ineq}) and use it to derive the dimension-dependent range of the constant learning rate that guarantees the contraction.

We first impose the following assumptions on the objective function and the stochastic gradients.

\begin{assumption}[Coercivity]\label{asm_coercive}
Assume that for any sequence $\bbeta_1,\bbeta_2,\ldots$ with $|\bbeta_n|_s\to \infty$ the loss function $G(\cdot)$ in~\eqref{eq_goal_sgd} satisfies $\lim_{n\rightarrow\infty}G(\bbeta_n)=\infty$.
\end{assumption}

\begin{assumption}[Strong Convexity -- $\ell^s$-norm]\label{asm_Ls_strong_convex}
Let $s\ge2$ be an even integer and write $\bm{v}^{\odot s}:=(v_1^s,\ldots,v_d^s)^{\top}$ for a vector $\bm{v}=(v_1,\ldots,v_d)^{\top}$. Assume there exists $\mu>0$ such that
\begin{equation*}
    \big\langle (\bbeta-\bbeta')^{\odot (s-1)}, \nabla G(\bbeta) - \nabla G(\bbeta') \big\rangle \ge \mu |\bbeta-\bbeta'|_s^s, \quad \text{for all} \ \ \bbeta,\bbeta'\in\RR^d.
\end{equation*}
\end{assumption}
In Lemma~\ref{lemma_global_beta} in the supplementary materials, we show that under Assumptions~\ref{asm_coercive}~and~\ref{asm_Ls_strong_convex}, a unique global minimizer $\bbeta^*$ exists for the optimization problem~\eqref{eq_goal_sgd}. When $s=2$, Assumption~\ref{asm_Ls_strong_convex} reduces to the regular strong convexity frequently adopted in the literature \citep{polyak_acceleration_1992,moulines_non-asymptotic_2011,dieuleveut_bridging_2020,mies_sequential_2023}. For general $s$ and the linear regression model, Section~\ref{subsec_linear} in the supplementary material interprets the $\ell^s$-type strong convexity assumption via the $\ell^s$-norm induced matrix norm. As different norms are involved, there does not seem to be an apparent relationship between the classical strong convexity and the case $s>2.$

\begin{assumption}[Stochastic Lipschitz Continuity -- $\ell^s$-norm]\label{asm_Ls_lip}
Let $\bbeta^*$ be the global minimizer. For some $q\ge2$ and an even integer $s\ge2$, assume that
\begin{equation*}
     M_{s,q}:=\big(\EE|\nabla g(\bbeta^*,\bm{\xi})|_s^q\big)^{1/q} <\infty.
\end{equation*}
Further assume there exists a constant $L_{s,q}>0$ such that
\begin{equation*}
    \Big(\EE\big|\nabla g(\bbeta,\bm{\xi}) - \nabla g(\bbeta',\bm{\xi})\big|_s^q\Big)^{1/q} \le L_{s,q}|\bbeta - \bbeta'|_s,
    \quad \text{for all} \ \ \bbeta,\bbeta'\in\RR^d.
\end{equation*}
\end{assumption}

Later we will choose $s=O(\log(d))$ to bound the max-norm. The above defined Lipschitz constant $L_{s,q}$ and the moments $M_{s,q}$ will then grow as $d$ increases. Taking linear regression as an example, we investigate the dimension dependence of $L_{s,q}$ and $M_{s,q}$ in Section~\ref{subsec_linear}. All bounds in this work will contain the explicit dependence on $(L_{s,q},M_{s,q})$.

We now state the first main result of this paper, which plays a crucial role in establishing moment convergence and tail probability results in the following sections. The statement quantifies the exponential rate at which the initialization $\bbeta_0$ will be forgotten and the SGD iterates $\bbeta_k$ converges to a stationary distribution $\pi_{\alpha}.$

\begin{theorem}[Convergence of SGD to stationary distribution]\label{thm_gmc}
Suppose that Assumptions~\ref{asm_coercive}--\ref{asm_Ls_lip} hold for some $\mu>0$, $q\ge2$ and even integer $s\ge2$. Given a constant learning rate  
\begin{equation}
    \label{eq_Ls_alpha_range}
   0 < \alpha < \alpha_{s,q}:=\frac{2\mu}{\max\{q,s\}L_{s,q}^2},
\end{equation}
for any two $d$-dimensional SGD sequences $\{\bbeta_k(\alpha)\}_{k\in\NN}$ and $\{\bbeta_k'(\alpha)\}_{k\in\NN}$ sharing the same i.i.d. noise injections $\{\bm{\xi}_k\}_{k\ge1}$ but possibly different initializations $\bbeta_0,\bbeta_0'\in\RR^d$, the geometric-moment contraction (GMC) \begin{equation*}
    \||\bbeta_k-\bbeta_k'|_s\|_q \le r_{\alpha,s,q}^k|\bbeta_0-\bbeta_0'|_s, \quad \text{for all} \ \ k=0,1,\ldots
\end{equation*}
holds with contraction constant 
\begin{equation}
    \label{eq_gmc}
    r_{\alpha,s,q} = 1-2\mu\alpha+\max\{q, s\}L_{s,q}^2\alpha^2 <1.
\end{equation}
Moreover, there exists a unique stationary distribution $\pi_{\alpha}$ with a finite $q$-th moment, that is, $\int|\bm{u}|_s^q\pi_{\alpha}(d\bm{u})<\infty$, such that
\begin{equation*}
    \bbeta_k\Rightarrow\pi_{\alpha}, \quad \text{as } \ k\rightarrow\infty.
\end{equation*}
Equivalently, for any continuous function $f\in\mathcal{C}(\RR^d)$ with $|f|_{\infty}<\infty$,
\begin{equation*}
    \EE\big[f(\bbeta_k)\big] \rightarrow \int f(\bm{u})\pi_{\alpha}(d\bm{u}), \quad \text{as }k\rightarrow\infty.
\end{equation*}
\end{theorem}
The result generalizes \citet{Lij_2024} to large dimension $d$ and extends the $\ell^2$-type GMC based on Lemma~\ref{lemma_moment_inequality} to general $\ell^s$-norms. Moreover, choosing $s=s_d$ with 
\begin{equation}
    \label{eq_s_choice}
    s_d:=2\min\{\ell\in\NN:\,2\ell>\log(d)\},
\end{equation}
and using the inequality
\begin{align}
    \label{eq_ls_linf_equiv_ineq}
    |\bm{x}|_{\infty} \le |\bm{x}|_{s_d} \le d^{1/s_d}|\bm{x}|_{\infty} \le e|\bm{x}|_{\infty},
\end{align}
shows the equivalence of the $\ell^{s_d}$- and $\ell^{\infty}$-norms. Consequently, by choosing $s=s_d,$ the previous theorem can also be used to derive the GMC property with respect to the $\ell^{\infty}$-norm.

\section{Convergence of High-Dimensional SGD and ASGD}
\label{sec_moment}
In this section, we derive convergence rates for the moments of the last iterate $\EE |\bbeta_k-\bbeta^*|_{\infty}^{q}$ and the moments of the averaged SGD.

\subsection{Convergence of SGD}
\label{subsec_moment_sgd}

\begin{proposition}
\label{prop_rio_sgd}
If Assumptions~\ref{asm_coercive}--\ref{asm_Ls_lip} hold for some $q\ge2,$ an even integer $s\ge2$, and a constant $M_{s,q}$, then,
\begin{align*}
    \big\||\bbeta_k-\bbeta^*|_s\big\|_q^2 & \le \big(1-2\alpha\mu+7\max\{q,s\}\alpha^2L_{s,q}^2\big)\big\||\bbeta_{k-1}-\bbeta^*|_s\big\|_q^2 + 3\max\{q,s\}\alpha^2M_{s,q}^2,
\end{align*}
for all $k\ge1.$ The same inequality holds if $\bbeta_k$ is replaced by the stationary SGD iterates $\bbeta_k^{\circ}\sim\pi_{\alpha}$, $k\ge1$.
\end{proposition}

\begin{theorem}[Moment convergence of SGD]\label{thm_sgd_moment}
Let $0<\alpha<\alpha_{s,q}/7$ with $\alpha_{s,q}$ as defined in~\eqref{eq_Ls_alpha_range}. Suppose that Assumptions~\ref{asm_coercive}--\ref{asm_Ls_lip} hold for $q\ge2$ and even integer $s\ge2$. Then for the stationary SGD iterates $\bbeta_k^{\circ}\sim\pi_{\alpha}$,
\begin{equation*}
    \big\||\bbeta_k^{\circ}-\bbeta^*|_s\big\|_q = O\Big(M_{s,q}\sqrt{\max\{q,s\}\alpha}\Big) \quad \text{for all} \ k\geq 1
\end{equation*}
and for the SGD iterate $\bbeta_k$ with arbitrary initialization $\bbeta_0$, 
\begin{align*}
    \big\||\bbeta_k-\bbeta^*|_s\big\|_q = O\Big(M_{s,q}\sqrt{\max\{q,s\}\alpha} + r_{\alpha,s,q}^k\||\bbeta_0-\bbeta_0^{\circ}|_s\|_q\Big) \quad \text{for all} \ k\geq 1.
\end{align*}
\end{theorem}
Choosing $s=s_d$ in~\eqref{eq_s_choice} yields a bound with respect to the $\ell^{\infty}$-norm.



\subsection{Convergence of Ruppert-Polyak Averaged SGD}
\label{subsec_moment_asgd}

Consider now the Ruppert-Polyak Averaged SGD (ASGD) $\bar\bbeta_k = \frac{1}{k}\sum_{i=1}^k\bbeta_i$ as defined in~\eqref{eq_asgd}. For the initialization $\bbeta_0^{\circ}\sim\pi_{\alpha}$, define the stationary ASGD sequence 
\begin{equation}
    \label{eq_asgd_stationary}
    \bar\bbeta_k^{\circ} = \frac{1}{k}\sum_{i=1}^k\bbeta_i^{\circ}, \quad k\ge1.
\end{equation}

\begin{theorem}
\label{Theorem_Moment_ASGD}
Consider the ASGD sequence $\{\bar\bbeta_k\}_{k\ge1}$. Suppose that Assumptions~\ref{asm_coercive}--\ref{asm_Ls_lip} hold with some $q\ge2$ and even integer $s=s_d$ in~\eqref{eq_s_choice}, the conditions of Theorem~\ref{thm_bias} hold and the learning rate satisfies $\alpha\in(0,\alpha_{s_d,q})$ with $\alpha_{s_d,q}$ defined in~\eqref{eq_Ls_alpha_range}. For any $k\ge1$ and some universal constants $C_1,C_2,C_3>0$,
\begin{align*}
    \big\||\bar\bbeta_k - \bbeta^*|_{\infty}\big\|_q & \le C_1\Bigg\{\underbrace{\sqrt{\frac{c_qs_d}{k}}M_{s_d,q}\Big(L_{s_d,q}\sqrt{\alpha\max\{q,s_d\}} + 1\Big)}_{\text{\small{stochastic variance}}}\Bigg\} \nonumber \\
    + \,& C_2\Big\{\underbrace{\frac{1}{k(1-r_{\alpha,s_d,q})}\||\bbeta_0-\bbeta_0^{\circ}|_{\infty}\|_q}_{\text{\small{initialization bias}}}\Big\} + C_3\Big\{\underbrace{M_{s_d,q}^2\max\{q,s_d\}\alpha d^{\frac{q}{q-1}\cdot(1-\frac{2}{s_d})}}_{\text{\small{bias of constant learning rate}}}\Big\}.
\end{align*}
\end{theorem}

\begin{proposition}[Complexity bound]\label{prop_complexity}
Under the assumptions of Theorem~\ref{Theorem_Moment_ASGD}, let $\Delta_{0} =\||\bbeta_{0}-\bbeta_{0}^{\circ}|_{\infty}\|_q$,
\[
V = L_{s_d,q}M_{s_d,q}\sqrt{\max\{q,s_d\}} + M_{s_d,q},
\,\,\,
B =M_{s_d,q}^2\max\{q,s_d\}\,d^{\frac{q}{q-1}\,(1-\frac{2}{s_d})}.
\]
Given a tolerance \(\varepsilon>0\), 
\[
\alpha \le \min\Big\{\frac{\varepsilon}{3C_3B},\frac{\alpha_{s_d,q}}{7}\Big\},
\quad \text{and} \ \ 
k \ge 
\max\Big\{\,
\frac{9C_1^2c_{q}s_dV^{2}\,\alpha}{\varepsilon^{2}},
\frac{3C_2\Delta_{0}}{\alpha\varepsilon}
\Big\},
\]
we have
\(
\||\bar\bbeta_k-\bbeta^{*}|_{\infty}\|_q\le\varepsilon.
\)
\end{proposition}
A proof outline is given in  Section~\ref{subsec_asgd_proof_sketch} and the full proof is deferred to the supplementary material. The sharpest complexity bound of SA for $\ell^{\infty}$-norm known to date was derived by \citet{wainwright_stochastic_2019} proving that the number of iterations required to obtain an $\varepsilon$-accurate solution of Q-learning scales as $(1-\gamma)^{-4}\cdot \varepsilon^{-2}$ with the discount factor $\gamma$. In Proposition~\ref{prop_complexity}, our complexity bound for SGD is also of the order of $O(1/\varepsilon^2)$ if the dimension $d$ is fixed, which is consistent with the degenerate Q-learning case in \citet{wainwright_stochastic_2019}. The derived result allows to determine the dependence on the dimension $d$.

\section{Sharp Concentration and Gaussian Approximation}\label{sec_concentration}
Via the following tail probability inequality for the averaged SGD estimator $\bar{\bbeta}_{k}$, one can further derive high-probability concentration bound of $|\bar{\bbeta}_{k}-\bbeta^*|_{\infty}$. Recall that $s_d=2\min\{\ell\in\NN:\,2\ell>\log(d)\}.$
\begin{theorem}[Fuk-Nagaev inequality]
\label{Theorem_Fuk_Nagaev}
    Under the conditions of Theorem~\ref{Theorem_Moment_ASGD}, for any $z > 0$, we have 
    \begin{align*}
        \mathbb{P} \big(|\bar{\bbeta}_{k} - \bbeta^{*}|_{\infty} > z\big) \lesssim \frac{ \||\bbeta_{0} - \bbeta_{0}^{\circ}|_{\infty}\|_q^{q}}{(k\alpha z)^{q}} + \frac{(\log d)^{\frac{3q}{2}} (\log k)^{1 + 2 q} M_{s_d, q}^{q}}{z^{q} k^{q - 1} \alpha^{q/2 - 1}} + \exp\left(- \frac{C k z^{2} \alpha^{1 - 2/q}}{M_{s_d, q}^{2} \log d}\right),
    \end{align*} 
    where the constants in $\lesssim$ are independent of $k,d,s$ and $\alpha$.
\end{theorem}

As an immediate consequence of Theorem~\ref{Theorem_Fuk_Nagaev}, we obtain a sharp high-probability upper bound for $|\bar{\bbeta}_{k} - \bbeta^{*}|_{\infty}$, that is, for any given tolerance rate $\delta \in (0,1)$, with at least probability $1-\delta$, we have  
\begin{align}
    \label{eq_prob_bound}
    |\bar{\bbeta}_{k} - \bbeta^{*}|_{\infty} = O\left(\frac{\||\bbeta_{0} - \bbeta_{0}^{\circ}|_{\infty}\|_q^{q}}{k\alpha \delta^{1/q}} + \frac{(\log d)^{3/2} (\log k)^{1/q + 2}M_{s_d,q}}{k^{1 - 1/q} \alpha^{1/2 - 1/q} \delta^{1/q}} + \sqrt{\frac{M_{s_d,q}^2\log d \log(1/\delta)}{k\alpha^{1-2/q}}}\right).  
\end{align}
Notably, if the~\emph{q}-th moment of the gradient noise is finite ($M_{s_d,q}<\infty$), the second term of the right hand side, involving $k^{1 - q}$, is generally unimprovable~\citep{nagaev_large_1979, lou_beyond_2022}.

The distribution convergence for the high-dimensional ASGD relies on the following result. Let $M_{2,q}$ be as defined in Assumption~\ref{asm_Ls_lip}.

\begin{theorem}[Gaussian approximation]\label{thm_GA}
Consider stationary SGD iterates $\boldsymbol{\beta}_k^{\circ}\sim\pi_{\alpha}$ with $\pi_{\alpha}$ as defined in Theorem~\ref{thm_gmc}, initialization $\bbeta_0^{\circ}\sim\pi_{\alpha},$ and learning rate $\alpha\in(0,\alpha_{s_d,q})$.  Suppose that Assumptions~\ref{asm_coercive}--\ref{asm_Ls_lip} hold for some $q>2$. Then, on a potentially different probability space, and for a number of iterations $T$ satisfying $d\le cT$, where $c>0$ is some constant, there exist random vectors $\{\tilde{\boldsymbol{\beta}}_k\}_{k=1}^T\overset{\mathcal{D}}{=}\{\boldsymbol{\beta}_k^{\circ}\}_{k=1}^T$ and independent Gaussian random vectors $\{\bm{z}_k\}_{k=1}^T$ with mean zero and covariance matrix
\begin{align}
    \Xi=\sum_{k=-\infty}^{\infty}\mathrm{Cov}(\boldsymbol{\beta}_0^{\circ},\boldsymbol{\beta}_k^{\circ}),
\end{align}
such that
\begin{align}
    \label{eq_GA_rate}
    \Big(\mathbb{E}\max_{k\le T}\Big|\frac{1}{\sqrt{k}}\sum_{i=1}^k\big[(\tilde{\boldsymbol{\beta}}_i -\mathbb{E}[\boldsymbol{\beta}_1^{\circ}]) -\bm{z}_i\big]\Big|_2^2\Big)^{1/2} \le C_{\alpha,q}^*M_{2,q}\sqrt{d\log(T)}\Big(\frac d T\Big)^{\frac{q-2}{6q-4}},
\end{align}
with $C_{\alpha,q}^*$ a constant that only depends on $c$, the learning rate $\alpha$, and the moment index $q.$  
\end{theorem}

For diverging moment index $q\rightarrow\infty$, the Gaussian approximation rate in~\eqref{eq_GA_rate} approaches the rate $O(\sqrt{\log(T)}(d^4/T)^{1/6})$. Thus, to obtain a nontrivial Gaussian approximation bound within $T$ iterations, we need dimension dependence $d=o(T^{1/4-\zeta})$ with $\zeta>0$.

\section{Constant Learning Rate for Large Dimension}\label{subsec_learning_rate}

Recall that $L_{s,q}$ is the Lipschitz constant introduced in Assumption~\ref{asm_Ls_lip}. We established asymptotic stationarity and non-asymptotic convergence if $\alpha<\alpha_{s,q}/7$ with $\alpha_{s,q}$ defined in~\eqref{eq_Ls_alpha_range}, leading to the upper bound
$$\alpha<\frac{\alpha_{s,q}}7 = \frac{2\mu}{7\max\{q,s\}L_{s,q}^2}\asymp\frac{1}{d^2\log(d)},$$
if we choose $s=s_d$ in~\eqref{eq_s_choice} and if $L_{s_d,q}\asymp d$. We refer to Section~\ref{subsec_linear} for the derivation of the dimension dependence of $L_{s,q}$ in the linear regression model. 

Alternatively, the upper bound for the learning rate $\alpha$ can also be derived by a linear approximation technique (see Lemma~\ref{lemma_Ls_moment_ineq}), defined as the nontrivial solution to the following equation
\begin{equation}
    1-q\mu\alpha + \frac{q\big[|q-s| + (s-1)\big]L_{s,q}^2}{2}\alpha^2 (1+\alpha L_{s,q})^{q-2} = 1.
    \label{eq.b9fe}
\end{equation} 
A derivation of this equation is provided in Section~\ref{subsec_ls_linf}. The existence of a solution of~\eqref{eq.b9fe} is shown below the proof of Lemma~\ref{lemma_Ls_moment_ineq} in the supplementary materials. When $q=2$, the range of $\alpha$ simplifies to 
\begin{equation*}
    \alpha < \frac{2\mu}{7\big[|s-2| + (s-1)\big]L_{s,2}^2},
\end{equation*}
which is also proportional to $1/[d^2\log(d)]$ if we choose $s=s_d$ in~\eqref{eq_s_choice} and if $L_{s_d,2}\asymp d$, matching the rate of $\alpha_{s,q}$ in~\eqref{eq_Ls_alpha_range} derived by Lemma~\ref{lemma_rio_Ls_moment_ineq}, though with a slightly more conservative constant for general $s$. In the special case with $s=2$, both bounds reduce to the classical $\alpha<2\mu/L_{2,2}^2$. If $L_{2,2}\asymp d$ for large dimension $d$, which is shown to be true for the linear regression model in Section~\ref{subsec_linear} in the supplementary materials, 
the $\ell^{\infty}$- and the $\ell^2$-norm yield similar upper bounds for the learning rate $\alpha.$

\section{Proof Sketches}\label{sec_proof_sketch}

\subsection{Bridge between \texorpdfstring{$\ell^s$}{l\^s}- and \texorpdfstring{$\ell^{\infty}$}{l\^inf}- Norms}\label{subsec_ls_linf}
In high-dimensional regimes, convergence rates of constant-learning-rate SGD~\eqref{eq_sgd_recursion} with respect to the $\ell^{\infty}$-norm are of particular interest \citep{wainwright_stochastic_2019,chen_concentration_2024}. However, it is extremely challenging to directly study the convergence of $|\bbeta_k-\bbeta^*|_{\infty}$  since the $\ell^{\infty}$-norm is not differentiable thereby ruling out standard gradient-based tools for proving convergence rates or concentration. To address this issue, we instead study $|\cdot|_{s_d}$ with $s_d$ defined in~\eqref{eq_s_choice}. By the equivalence between $\ell^{s_d}$- and $\ell^{\infty}$-norms shown in~\eqref{eq_ls_linf_equiv_ineq},
contraction in $\ell^{\infty}$-norm follows from $\ell^{s_d}$-norm contraction.

To prove the GMC property of SGD as introduced in~\eqref{eq_gmc_intro}, it suffices to show that for any two $d$-dimensional SGD sequences $\{\bbeta_k\}_{k\in\NN}$ and $\{\bbeta_k'\}_{k\in\NN}$ sharing the same i.i.d. observations $\{\bm{\xi}_k\}_{k\ge1}$ but possibly different initializations $\bbeta_0,\bbeta_0'\in\RR^d$, the contraction holds for $|\bbeta_k-\bbeta'_k|_{s_d}$ for all $k\ge1$. To this end, we need to determine a range of constant learning rates $\alpha$ such that for any $q\ge2$ and $\bbeta,\bbeta'\in\RR^d$, the GMC in Theorem~\ref{thm_gmc} holds, i.e.,
\begin{align}
    \label{eq_gmc_ls}
    \big(\EE\big|\bbeta - \alpha\nabla g(\bbeta,\bm{\xi}) - \big(\bbeta' - \alpha\nabla g(\bbeta',\bm{\xi})\big)\big|_s^q\big)^{1/q} \le r|\bbeta - \bbeta'|_s, \quad \text{for some } r=r_{\alpha,s,q}<1.
\end{align}
To derive the inequality, we first provide a lemma based on linear approximation by considering the scalar function 
\begin{equation*}
    \alpha\mapsto|\bm{x}-\alpha\bm{z}|_s^q, \quad \text{where } \bm{x} = \bbeta - \bbeta', \quad \bm{z} = \nabla g(\bbeta,\bm{\xi}) - \nabla g(\bbeta',\bm{\xi}),
\end{equation*}
and linearizing it around $\alpha=0$. Then, one only needs to prove that $\EE|\bm{x} - \alpha\bm{z}|_s^q \le r|\bm{x}|_s^q$.
By the second-order Taylor expansion of $|\bm{x}-\alpha\bm{z}|_s^q$ in $\alpha$, we have the linear approximation 
\begin{equation}
    |\bm{x}-\alpha\bm{z}|_s^q \approx |\bm{x}|_s^q - q\alpha|\bm{x}|_s^{q-s}\langle\bm{x}^{s-1},\bm{z}\rangle,
    \label{eq.9bdw8}
\end{equation}
with remainder term of order $\alpha^2$, see Section~\ref{sec_gmc} in the supplementary materials for details. Since a simple triangle inequality argument $\||\bm{x}-\alpha\bm{z}|_s\|_q\le \||\bm{x}|_s\|_q + \alpha\||\bm{z}|_s\|_q$ fails to control this remainder sufficiently to yield a contraction constant $r<1$, we establish a more precise bound. 
\begin{lemma}
\label{lemma_Ls_moment_ineq}
Recall that $\bm{v}^{\otimes s}=(v_1^s,\ldots,v_d^s)^{\top}$ for a vector $\bm{v}=(v_1,\ldots,v_d)^{\top}$. For any $q\ge2$, any even integer $s\ge2$, any two vectors $\bm{x},\bm{z}\in \RR^d,$ and any $\alpha>0$,
\begin{align*}
    \Big| |\bm{x}-\alpha\bm{z}|_s^q - |\bm{x}|_s^q + q\alpha|\bm{x}|_s^{q-s}\langle\bm{x}^{s-1},\bm{z}\rangle \Big| \le \frac{q\alpha^2}{2}\big[|q-s| + (s-1)\big]\big(|\bm{x}|_s + \alpha|\bm{z}|_s\big)^{q-2}|\bm{z}|_s^2.
\end{align*}
\end{lemma}
If $s=2$, $q=2$, the right-hand side is $\alpha^2|\bm{z}|_2^2$. This is consistent with the Taylor remainder of the right-hand side in Lemma~\ref{lemma_moment_inequality} derived by \citet{li_asymptotics_2024}. Using this inequality to prove the contraction in~\eqref{eq_gmc_ls} is remarkably different from the approaches relying on the martingale decomposition (MD) that is frequently adopted in the literature \citep{dieuleveut_bridging_2020,2020_Neurips_SGD_Martikopolous,mies_sequential_2023}. Our proposed method requires mild moment conditions on the stochastic gradients and yields simpler proofs that can be generalized to a broad class of online learning problems. We refer to \citet{Lij_2024} for detailed discussion. Nevertheless, we remark in advance that a Rio-type inequality (Lemma~\ref{lemma_rio_Ls_moment_ineq}) with slightly sharper constants will be used directly in our main contraction proof, while we retain Lemma~\ref{lemma_Ls_moment_ineq} here for its intuitive appeal. Finally, by choosing $s=s_d$ as in~\eqref{eq_s_choice} for~\eqref{eq_gmc_ls}, we can expect the $\ell^{\infty}$-norm type GMC to hold for high-dimensional SGD iterates.

\subsection{High-Dimensional Moment Inequality}\label{subsec_hd_ineq}

To prove Theorem~\ref{thm_gmc}, we derive a high-dimensional version of Rio's inequality \citep{rio_moment_2009}, adapted to the $q$-th moment of $\ell^s$-norm. This result provides a slightly sharper constant than Lemma~\ref{lemma_Ls_moment_ineq} and is used directly in our moment-contraction analysis. 

\begin{lemma}[High-dimensional moment inequality]\label{lemma_rio_Ls_moment_ineq}
For any $q\ge2$, any even integer $s\ge2$, and any two $d$-dimensional random vectors $\bm{x}, \bm{y}$, we have
\begin{align*}
    \big\||\bm{x} + \bm{y}|_{s}\big\|_{q}^{2} 
    \leq \big\||\bm{x}|_{s}\big\|_{q}^{2} + 2 \big\||\bm{x}|_{s}\big\|_{q}^{2 - q} \mathbb{E} \Big(|\bm{x}|_{s}^{q - s} \sum_{j = 1}^{d} x_{j}^{s - 1} y_{j}\Big) + \big(\max\{q, s\} - 1\big) \big\||\bm{y}|_{s}\big\|_{q}^{2}. 
\end{align*}
Moreover, if $\EE[\bm{y}\mid \bm{x}]=0$ almost surely, then
\begin{align}
    \label{eq_hd_rio}
    \big\||\bm{x} + \bm{y}|_{s}\big\|_{q}^{2} 
    \leq \big\||\bm{x}|_{s}\big\|_{q}^{2} + \big(\max\{q, s\} - 1\big) \big\||\bm{y}|_{s}\big\|_{q}^{2}. 
\end{align}
\end{lemma}

Repeatedly applying Lemma~\ref{lemma_rio_Ls_moment_ineq} leads to the high-dimensional maximal moment inequality in Lemma~\ref{lemma_general_chernozhukov} in the supplementary materials, which is of independent interest.

\subsection{Stationarity, Variation and Bias of ASGD}\label{subsec_asgd_proof_sketch}
We prove the moment bound $\| |\bar\bbeta_k-\bbeta^*|_{\infty}\|_q$ via the decomposition
\begin{align*}
    \big\| |\bar\bbeta_k-\bbeta^*|_{\infty}\big\|_q \le \big\||\bar\bbeta_k - \bar\bbeta_k^{\circ}|_{\infty}\big\|_q + \big\||\bar\bbeta_k^{\circ} - \EE[\bar\bbeta_k^{\circ}]|_{\infty}\big\|_q + \big|\EE[\bar\bbeta_k^{\circ}]-\bbeta^*\big|_{\infty}.
\end{align*}
The first term accounts for the deviation due to the non-stationarity of $\bar\bbeta_k$ as it is initialized from an arbitrarily fixed $\bbeta_0$; this can be bounded using the GMC property of $\bbeta_k$ shown in Theorem~\ref{thm_gmc}. The second term captures the stochastic variance of the stationary ASGD sequence. Bounding this term is more delicate because of the intricate dependency structure of $\bar\bbeta_k^{\circ}$. To address this, we deploy another powerful tool in time series -- the \textit{functional dependence measure} \citep{wu_nonlinear_2005} in Section~\ref{subsubsec_functional_dep} of the supplementary materials, which can effectively quantify the contribution of the random sample $\bm{\xi}_i$ to the $k$-th SGD iterate $\bbeta_k^{\circ}$ for all $i\le k$. As such, by controlling the cumulative dependence measures, we can bound this variance. Lastly, we handle the third term, which represents the non-diminishing bias of $\bar\bbeta_k^{\circ}$ induced by the constant learning rate $\alpha$ \citep{dieuleveut_bridging_2020,huo_bias_2023}. This can be dealt with by extending the approach in \citet{Lij_2024} to high-dimensional settings.

\begin{theorem}[Asymptotic stationarity]
\label{thm_asgd_stationary}
Consider the ASGD iterates $\bar\bbeta_k$ and the stationary version $\bar\bbeta_k^{\circ}$. Suppose that Assumptions~\ref{asm_coercive}--\ref{asm_Ls_lip} are satisfied for some $q\ge2$ and some even integer $s\ge2$. Then, for the learning rate $\alpha\in(0,\alpha_{s,q})$ with $\alpha_{s,q}$ defined in~\eqref{eq_Ls_alpha_range},
\begin{align*}
    \big\||\bar\bbeta_k - \bar\bbeta_k^{\circ}|_s\big\|_q \le \frac{1}{k}\cdot\frac{1}{1-r_{\alpha,s,q}}|\bbeta_0-\bbeta_0^{\circ}|_s.
\end{align*}
\end{theorem}

As a direct consequence of Theorem~\ref{thm_asgd_stationary}, we have $\||\bar\bbeta_k-\bar\bbeta_k^{\circ}|_s\|_q \lesssim |\bbeta_0-\bbeta_0^{\circ}|_s / (k\alpha) $,
which indicates the asymptotic stationarity of high-dimensional ASGD sequences. When the bias induced by the initialization is controlled, i.e., $|\bbeta_0-\bbeta_0^{\circ}|_s<\infty$, as $k\alpha\rightarrow\infty$, the ASGD iterate $\bar\bbeta_k$ approaches the stationary solution $\bar\bbeta_k^{\circ}$ in the sense that $\||\bar\bbeta_k-\bar\bbeta_k^{\circ}|_s\|_q  \rightarrow 0$. By Theorem~\ref{thm_asgd_stationary}, we only need to show the convergence for stationary ASGD.

\begin{theorem}[Stochasticity of stationary ASGD]
\label{thm_asgd_stochastic}
Consider the stationary SGD sequence $\{\bbeta_k^{\circ}\}_{k\ge1}$. Suppose that Assumptions~\ref{asm_coercive}--\ref{asm_Ls_lip} hold with some $q\ge2$ and some even integer $s\ge2$. Then there exists a constant $c_q>0$ only depending on $q$, such that, for all $k\ge1$,
\begin{equation*}
    \big\||\bar\bbeta_k^{\circ} - \EE[\bar\bbeta_k^{\circ}]|_s\big\|_q \le \sqrt{\frac{c_qs}{k}}M_{s,q}\Big(L_{s,q}\sqrt{\alpha \max\{q,s\}} + 1\Big).
\end{equation*}
\end{theorem}

In the low-dimensional case, we take $s=2$ as a special example. Then, $L_{s,q}\sqrt{\alpha \max\{q,s\}}=O(1)$ such that the bound is $\||\bar\bbeta_k^{\circ} - \EE[\bar\bbeta_k^{\circ}]|_s\|_q = O\{1/\sqrt{k}\}$. This rate is optimal considering the central limit theorem of the stationary ASGD.

Next, we consider the bias induced by the constant learning rate. We first introduce some necessary notation. Recall $G(\bbeta)=\EE[\nabla g(\bbeta,\bm{\xi})]$ and $\nabla G(\bbeta) = (\partial_1 G(\bbeta),\ldots,\partial_d G(\bbeta))^{\top}$, where $\bbeta=(\beta_1,\ldots,\beta_d)^{\top}$. Denote $\partial_i G(\bbeta) = \partial G(\bbeta)/\partial\beta_i$, $1 \le i \le d$,
\begin{align}
    \label{eq_partials_G}
    \nabla^2G(\bbeta) = \big[\partial_i\partial_jG(\bbeta)\big]_{1\le i,j\le d}, \quad \nabla^3 G_i(\bbeta) =  \big[\partial_i\partial_l\partial_r G(\bbeta)\big]_{1\le l,r\le d}. 
\end{align}
We provide the non-asymptotic bound for the bias of stationary ASGD in the following lemma.

\begin{theorem}[Bias of stationary ASGD]
\label{thm_bias}
Under Assumptions~\ref{asm_coercive}--\ref{asm_Ls_lip}, consider the stationary ASGD $\bar\bbeta_k^{\circ}$. Assume that $g(\bbeta,\xi)$ is twice differentiable with respect to $\bbeta$ with positive definite Hessian matrix $\nabla^2 G(\bbeta^*)$, and uniformly bounded derivatives $\max_{1\le i\le d}\|\nabla^3G_i(\bbeta)\|_{\infty}<\infty$,
where
\begin{equation*}
    \|\nabla^3 G_i(\bbeta)\|_{\infty}:=\max_{1\le l\le d} \, \sum_{r=1}^d\Big|\big(\nabla^3 G_i(\bbeta)\big)_{l,r}\Big|.
\end{equation*}
Then, we have 
\begin{align*}
    \big|\EE[\bar\bbeta_k^{\circ}-\bbeta^*]\big|_{\infty} = O\Big(M_{s_d,q}^2\max\{q,s_d\}\alpha d^{\frac{q}{q-1}\cdot(1-\frac{2}{s_d})}\Big).
\end{align*}
\end{theorem}

\section{Conclusions and Discussion}
This work advances the theoretical understanding of the constant learning-rate SGD algorithms in high-dimensional settings. By introducing novel coupling techniques in nonlinear time series, we establish asymptotic stationarity of SGD with any initialization. We then derive non-asymptotic $q$-th moment bounds in general $\ell^s$- and $\ell^\infty$-norms, and develop the first Fuk-Nagaev high-probability tail bound for ASGD. While this paper assumes strong convexity and smoothness of the objective, the nonlinear time series perspective offers a principled framework applicable to a broad class of over-parameterized optimization tasks and can be extended to non-convex regimes, providing fundamental insights into the stability, convergence, and reliability of large-scale learning algorithms.

\begin{ack}
We sincerely thank the program chair, senior area chair, area chair, and the four reviewers for their constructive feedback and involved discussion, which has greatly improved the clarity of our paper.
Jiaqi Li’s research is partially supported by the NSF (Grant NSF/DMS-2515926). Johannes Schmidt-Hieber has received funding from the Dutch Research Council (NWO) via the Vidi grant VI.Vidi.192.021. Wei Biao Wu’s research is partially supported by the NSF (Grants NSF/DMS-2311249, NSF/DMS-2027723). We would like to thank Insung Kong for helpful discussions.

\end{ack}

\bibliographystyle{abbrvnat}
\bibliography{NeurIPS2025/neurips_2025}

\newpage
\section{Technical Appendices and Supplementary Material}


\subsection{Existence and Uniqueness of Global Minimum}

\begin{lemma}\label{lemma_global_beta}
    Consider the minimization problem $\bbeta^*\in\arg\min_{\bbeta\in\RR^d}G(\bbeta)$. If the function $G$ satisfies Assumptions~\ref{asm_coercive}~and~\ref{asm_Ls_strong_convex}, then a global minimizer $\bbeta^*$ exists and is unique.
\end{lemma}
\begin{proof}[Proof of Lemma~\ref{lemma_global_beta}]
We first show the existence of a global minimizer. By the coercivity condition in Assumption~\ref{asm_coercive}, $\lim_{|\bbeta|_s\rightarrow\infty}G(\bbeta) = \infty$, which implies that we can choose some large $\delta\in\RR$ such that the sub-level set
\begin{equation*}
    \S_{\delta} := \{\bbeta\in\RR^d: \, G(\bbeta)\le \delta\} 
\end{equation*}
is non-empty and bounded. Since $G$ is continuous by Assumption~\ref{asm_Ls_strong_convex}, $\S_{\delta}$ is also closed, and hence compact in $\RR^d$ by the Heine–Borel theorem. Finally, by applying the Weierstrass extreme value theorem, there exists $\bbeta^*\in\S_{\delta}$ such that $G(\bbeta^*)=\min_{\bbeta\in\S_{\delta}}G(\bbeta)$. Since for any $\bbeta\notin\S_{\delta}$, $G(\bbeta) >\delta\ge G(\bbeta^*)$, $G(\bbeta^*)=\min_{\bbeta\in\RR^d}G(\bbeta)$.

Next, we show the uniqueness of the global minium. Assume that there are two distinct minimizers $\bbeta_1\neq\bbeta_2$. By Assumption~\ref{asm_Ls_strong_convex}, there exists $\mu>0$ such that
\begin{align*}
    \langle (\bbeta_1 - \bbeta_2)^{\odot(s-1)}, \nabla G(\bbeta_1) - \nabla G(\bbeta_2) \rangle \ge \mu|\bbeta_1-\bbeta_2|_s^s >0.
\end{align*}
However, since $\bbeta_1$ and $\bbeta_2$ are both minimizers, $\nabla G(\bbeta_1)=\nabla G(\bbeta_2)=0$, while $\mu|\bbeta_1-\bbeta_2|_s^s>0$. This leads to contradiction, which finishes the proof.
\end{proof}

\subsection{Example: Linear Regression}\label{subsec_linear}

As example, we consider the SGD algorithm for the high-dimensional linear regression, observing independent and identically distributed (i.i.d.) pairs $\bm{\xi}_1:=(\bm{x}_1,y_1),\bm{\xi}_2:=(\bm{x}_2,y_2),\ldots$ satisfying 
\begin{equation}
   \label{eq_linear_model}
   y_k=\bm{x}_k^{\top}\bbeta + \epsilon_k, \quad \text{for} \ \ k=1,2,\ldots, 
\end{equation}
for random noises $\epsilon_k$ that are independent of $\bm{x}_k$ with $\EE[\epsilon_k]=0$ and $\EE|\epsilon_k|^q<\infty$ for some $q\ge2$.
We verify Assumptions \ref{asm_Ls_strong_convex} and \ref{asm_Ls_lip} and derive the explicit dependency of the learning-rate, the Lipschitz constant, and the moments of the gradient noise on the dimension $d$.

Let $\bm{\xi}=(y,\bm{x})$ be an independent random sample from the same distribution as the data. The least-squares loss and the stochastic gradient are respectively given by
\begin{equation}
    g(\bbeta,\bm{\xi}) = \frac{1}{2}(y-\bm{x}^{\top}\bbeta)^2, \quad \text{and} \ \  \nabla g(\bbeta,\bm{\xi}) = -(y-\bm{x}^{\top}\bbeta)\bm{x}.
\end{equation}
Then
\begin{align}
    \nabla G(\bbeta) = \EE[\nabla g(\bbeta,(y,\bm{x})] = -\EE[(y-\bm{x}^{\top}\bbeta)\bm{x}] = \EE[\bm{x}\bm{x}^{\top}](\bbeta-\bbeta^*).
\end{align}
Let
\begin{equation}
    \Sigma = \EE[\bm{x}\bm{x}^{\top}], \quad \bm{v} = \bbeta-\bbeta'.
\end{equation}

To verify the $\ell^s$-type strong convexity 
\begin{align*}
    \big\langle (\bbeta_1 - \bbeta_2)^{\odot(s-1)}, \nabla G(\bbeta) - \nabla G(\bbeta') \big\rangle \ge \mu |\bbeta-\bbeta'|_s^s, \quad \text{for all} \ \ \bbeta,\bbeta'\in\RR^d,
\end{align*}
imposed in Assumption \ref{asm_Ls_strong_convex},
observe that $\nabla G(\bbeta)-\nabla G(\bbeta')=\Sigma\bm{v}$. Thus, the condition becomes
\begin{equation}
    \label{eq_min_eigen}
    0< \lambda_{\min}^{(s)} := \inf_{\bm{v}\in\RR^d,\bm{v}\neq0}\frac{\langle \bm{v}^{s-1},\Sigma\bm{v}\rangle}{|\bm{v}|_s^s}.
\end{equation}

\begin{lemma}\label{lemma_min_eigen}
Let $s\in \{2,4,6,\ldots\}.$ Writing $\Sigma=(\Sigma_{ij})_{i,j=1,\ldots,d},$ we have 
\begin{align*}
    \lambda_{\min}^{(s)}
    \geq \min_{i=1,\ldots,d} \, \Sigma_{ii} - \sum_{j : j\neq i} |\Sigma_{ij}|.
\end{align*}
\end{lemma}

\begin{proof}[Proof of Lemma~\ref{lemma_min_eigen}]
Write $\bm{v}=(v_1,\ldots,v_d)^\top.$ Because of $(|v_i|^{s-1}-|v_j|^{s-1})(|v_i|-|v_j|)\geq 0,$ we obtain $|v_i^{s-1}v_j|+|v_iv_j^{s-1}|\leq v_i^s+v_j^s$ and
\begin{align*}
    \langle \bm{v}^{\odot(s-1)},\Sigma\bm{v}\rangle
    &=\sum_{i=1}^d \Sigma_{ii} v_i^s 
    - \sum_{i<j} \Sigma_{ij} \big(v_i^{s-1}v_j+v_i v_j^{s-1}\big) \\
    &\geq 
    \sum_{i=1}^d \Sigma_{ii} v_i^s 
    - \sum_{i<j} \big|\Sigma_{ij}\big| \big(v_i^s+v_j^s\big) \\
    &\geq \Big(\min_{i=1,\ldots,d} \, \Sigma_{ii} - \sum_{j : j\neq i} |\Sigma_{ij}|
    \Big) \sum_{\ell} v_\ell^s.
\end{align*}
\end{proof}

This shows that the rightmost ``Gershgorin gap'' $\min_{i=1,\ldots,d} \, \Sigma_{ii} - \sum_{j : j\neq i} |\Sigma_{ij}|$ is a universal lower bound for every $s$. The lower bound is non-trivial if $\Sigma$ is sufficiently diagonally dominant. 

For large $s,$ the inequality $\lambda_{\min}^{(s)}
    \geq \min_{i=1,\ldots,d} \, \Sigma_{ii} - \sum_{j : j\neq i} |\Sigma_{ij}|$ is nearly sharp. To see this, let $i^*$ be the index $i$ that minimizes $\min_{i=1,\ldots,d} \, \Sigma_{ii} - \sum_{j : j\neq i} |\Sigma_{ij}|.$ For a small $\delta>0,$ pick $\bm{v}=(v_1,\ldots,v_d)$ by choosing $v_{i^*}:=1$ and for $i\neq i^*,$ taking $v_{i}:=-\operatorname{sign}(\Sigma_{i^*i})(1-\delta).$ For large $s,$ $\bm{v}^{\odot(s-1)}\approx (0,0,\ldots,1,0,\ldots,0)$ with the $1$ at the $i^*$-th position. Similarly, $|\bm{v}|_s^s \approx 1.$ The $i^*$-th entry of $\Sigma \bm{v}$ is given by $\Sigma_{i^*i^*}-\sum_{j\neq i^*} |\Sigma_{i^*j}| + O(\delta).$ Hence for suitable sequences $\delta\to 0$ and $s\to \infty$, we obtain $\langle \bm{v}^{\odot(s-1)},\Sigma\bm{v}\rangle/|\bm{v}|_s^s \to \Sigma_{i^*i^*}-\sum_{j\neq i^*} |\Sigma_{i^*j}|= \min_{i} \Sigma_{ii}-\sum_{j\neq i} |\Sigma_{ij}|.$




Regarding Assumption \ref{asm_Ls_lip}, we investigate the dependence of the Lipschitz constant $L_{s,q}$ on the dimension $d$ in high-dimensional linear regression models. If $s^*$ is the dual exponent of $s$, satisfying $1/s+1/s^*=1,$ we show that the condition holds with
\begin{equation}
    L_{s,q} = \big\| |\bm{x}|_s|\bm{x}|_{s^*}\big\|_q.
    \label{eq.9bfewe}
\end{equation}
To see this, for any two vectors $\bbeta,\bbeta'\in\RR^d$, we have
\begin{align}
    \nabla g(\bbeta,\bm{\xi}) - \nabla g(\bbeta',\bm{\xi}) & = -\Big[(y-\bm{x}^{\top}\bbeta)\bm{x} - (y-\bm{x}^{\top}\bbeta')\bm{x}\Big]  = \bm{x}\bm{x}^{\top}(\bbeta-\bbeta').
\end{align}
Taking the $\ell^s$-norm on both sides, we obtain
\begin{equation}
    \big|\nabla g(\bbeta,\bm{\xi}) - \nabla g(\bbeta',\bm{\xi})\big|_s = \big|\bm{x}\bm{x}^{\top}(\bbeta-\bbeta')\big|_s = |\bm{x}|_s\big|\bm{x}^{\top}(\bbeta-\bbeta')\big|.
\end{equation}
By Hölder's inequality, for the dual exponent $s^*$ satisfying $1/s+1/s^*=1$, it follows that
\begin{equation}
    \big|\bm{x}^{\top}(\bbeta-\bbeta')\big| \le |\bm{x}|_{s^*}|\bbeta-\bbeta'|_s.
\end{equation}
Therefore, for $q\ge2$, we have the $q$-th moment bounded as follows,
\begin{align*}
    \Big(\EE\big|\nabla g(\bbeta,\bm{\xi}) - \nabla g(\bbeta',\bm{\xi})\big|_s^q\Big)^{1/q} \le \Big(\EE\big[|\bm{x}|_s^q|\bm{x}|_{s^*}^q\big]\Big)^{1/q}|\bbeta-\bbeta'|_s,
\end{align*}
proving \eqref{eq.9bfewe}. 

Recall $s_d$ defined in~\eqref{eq_s_choice}. To bound the $\ell^{\infty}$-norm, we set the conjugates
\begin{equation}
    s=s_d,\quad s_d^*=\frac{s_d}{s_d-1}.
\end{equation}
Recall that for the $\ell^s$-norm, we have $|\bm{x}|_{\infty} \le |\bm{x}|_{s_d} \le d^{1/s_d}|\bm{x}|_{\infty} \le e|\bm{x}|_{\infty}$. Similarly, for the conjugate $\ell^{s_d^*}$-norm, $d^{\frac{1}{s_d^*}-1} =d^{\frac{1}{s_d}} \le e$ implies 
\begin{equation}
    \frac{1}{e}|\bm{x}|_1 \leq \frac{1}{d^{\frac{1}{s_d^*}-1}}|\bm{x}|_1 \le |\bm{x}|_{s_d^*} \le |\bm{x}|_1,
\end{equation}
which together with~\eqref{eq.9bfewe} gives
\begin{equation}
    \label{eq_Lsq_upper}
    L_{s_d,q} \le e\big\||\bm{x}|_{\infty}|\bm{x}|_1\big\|_q.
\end{equation}

The next two lemmas show that the tail behavior of the covariate vector $\bm{x}_k$ determines the behavior of the Lipschitz constant $L_{s,q}$ and the moment $M_{s,q}$ defined in Assumption~\ref{asm_Ls_lip}.

\begin{lemma}\label{lemma_lip} Consider the linear regression in~\eqref{eq_linear_model} with i.i.d.\ generic random samples $(\bm{x},y)$, where $\bm{x}=(x_1,\ldots,x_d)^{\top}$. Let $q\ge2$ and recall $s_d$ in~\eqref{eq_s_choice}.
\begin{itemize}
    \item [(i)] (Sub-Gaussian) If there is a constant $K$ such that for all $\bm{u}\in\RR^d$, $|\bm{u}^{\top}\bm{x}|_{\psi_2}\le K|\bm{u}|_2$, where $|v|_{\psi_2}=\inf\{t>0:\,\EE[e^{v^2/t^2}]\le2\}$ denotes the sub‐Gaussian norm, then \[L_{s_d,q}=O(d\sqrt{\log(d)}).\]
    \item[(ii)] (Sub-exponential) If there is a constant $K$ such that for all $\bm{u}\in\RR^d$, $|\bm{u}^{\top}\bm{x}|_{\psi_1}\le K|\bm{u}|_2$, where $|v|_{\psi_1}=\inf\{t>0:\,\EE[e^{|v|/t}]\le2\}$ denotes the sub‐exponential norm, then \[L_{s_d,q}=O(d\log(d)).\]
    \item[(iii)] (Finite moment) If there is some $p\ge 2q$ and a finite constant $K_p$ such that for each $1\le j\le d$, $\EE|x_j|^p\le K_p$, then \[L_{s_d,q}=O(d^{1+\frac{1}{2q}}).\]
    \item[(iv)] For all three cases (i)--(iii), when $s=2$, $L_{2,q}=O(d).$
\end{itemize}    
\end{lemma}

\begin{proof}[Proof of Lemma~\ref{lemma_lip}]
We write $\bm{x}=\bm{x}_k$ to denote a generic covariate. By~\eqref{eq_Lsq_upper} and Hölder's inequality,
\begin{align*}
    L_{s_d,q} \le e\big\||\bm{x}|_{\infty}|\bm{x}|_1\big\|_q \le e\||\bm{x}|_{\infty}\|_{2q}\||\bm{x}|_1\|_{2q}.
\end{align*}
The convexity of the function $t\mapsto t^{2q}$ and Jensen's inequality yield $|\bm{x}|_1^{2q} \le d^{2q-1}\sum_{j=1}^d|x_j|^{2q}$ and
\begin{align*}
    \EE|\bm{x}|_1^{2q} \le d^{2q-1}\sum_{j=1}^d\EE|x_j|^{2q} \le d^{2q}\max_{1\le j\le d}\EE|x_j|^{2q}. 
\end{align*}
Therefore, for all the three cases (i)--(iii),
\begin{equation*}
    \||\bm{x}|_1\|_{2q} = O(d).
\end{equation*}
Next, we study the order of $(\EE[|\bm{x}|_{\infty}^{2q}])^{1/(2q)}$ for fixed $q\ge2$.

(i) If each $x_j$ is sub-Gaussian, then by Section 2.5 in \citet{vershynin_high-dimensional_2018}, we have
\begin{align*}
    (\EE[\max_{1\le j\le d}|x_j|^{2q}])^{1/(2q)} \le K(\sqrt{\log(d)} + \sqrt{q}) = O(\sqrt{\log(d)}).
\end{align*}
(ii) If each $x_j$ is sub-exponential, then by Section 2.7 in \citet{vershynin_high-dimensional_2018}, we obtain
\begin{align*}
    (\EE[\max_{1\le j\le d}|x_j|^{2q}])^{1/(2q)} = O(K(\log(d) + \log(q))) = O(\log(d)).
\end{align*}
(iii) If each $x_j$ has the finite $p$-th moment for some $p\ge2q$, then 
\begin{align*}
    \EE[\max_{1\le j\le d}|x_j|^{2q}] \le \sum_{1\le j\le d}\EE[|x_j|^{2q}] \le dK_q=O(d).
\end{align*}

Finally, for case (iv) with $s_d=2$, by~\eqref{eq.9bfewe},
\begin{align*}
    L_{2,q} =\||\bm{x}|_2\|_{2q}^2.
\end{align*}
By the convexity of the function $t\mapsto t^{q}$, we apply Jensen's inequality and obtain
\begin{align*}
    |\bm{x}|_2^{2q} = \Big(\sum_{j=1}^dx_j^2\Big)^q \le d^{q-1}\sum_{j=1}^d|x_j|^{2q}.
\end{align*}
Therefore, for $\bm{x}$ satisfying case (iii),
\begin{align}
    \EE|\bm{x}|_2^{2q} \le d^{q-1}\sum_{j=1}^d\EE|x_j|^{2q} \le d^qK_p,
\end{align}
which yields $L_{2,q}=O(d)$. For the cases (i) and (ii), by Sections 3.4 and 2.7 in \citet{vershynin_high-dimensional_2018}, respectively, we obtain
\begin{align*}
    \||\bm{x}|_2\|_{2q}=O(K(\sqrt{d}+\sqrt{q})) = O(\sqrt{d}),
\end{align*}
and
\begin{align*}
    \||\bm{x}|_2\|_{2q}=O(K(q\sqrt{d})) = O(\sqrt{d}),
\end{align*}
both indicating $L_{2,q}=O(d)$. This completes the proof.
\end{proof}

\begin{lemma}\label{lemma_Msq}
Consider the linear regression model in~\eqref{eq_linear_model} and assume the conditions on $\epsilon$ and $\bm{x}$ therein are satisfied. Recall that $M_{s,q}=\||\nabla g(\bbeta^*,\bm{\xi})|_s\|_q$ is defined in Assumption~\ref{asm_Ls_lip} for some $q\ge2$. For the same four cases (i)--(iv) as in Lemma~\ref{lemma_lip} and $s_d$ defined in~\eqref{eq_s_choice}, $M_{s_d,q}$ is respectively equal to (i) $O(\sqrt{\log(d)})$, (ii) $O(\log(d))$, (iii) $O(d^{1/(2q)})$ and (iv) $O(\sqrt{d})$.
\end{lemma}
\begin{proof}[Proof of Lemma~\ref{lemma_Msq}]
In the linear regression model, the stochastic gradient at the global minimum $\bbeta^*$ can be rewritten into
\begin{equation*}
    \nabla g(\bbeta^*,\bm{\xi}) = -(y-\bm{x}^{\top}\bbeta^*) = -\epsilon\bm{x}.
\end{equation*}
Since the noise $\epsilon$ is independent of the covariate vector $\bm{x}$, we obtain
\begin{align*}
    \||\nabla g(\bbeta^*,\bm{\xi})|_{s_d}\|_q = \||\epsilon|\cdot|\bm{x}|_{s_d}\|_q = \|\epsilon\|_{q}\||\bm{x}|_{s_d}\|_{q}.
\end{align*}
By inequality~\eqref{eq_ls_linf_equiv_ineq}, it suffices to bound $\||\bm{x}|_{\infty}\|_{q}$. Since $\||\bm{x}|_{\infty}\|_{q} \le \||\bm{x}|_{\infty}\|_{2q}$, the same arguments in the proof of Lemma~\ref{lemma_lip} carry over immediately. We omit the details here.
\end{proof}

\subsection{Some Useful Lemmas}

\begin{lemma}[Maximal inequality \citep{chernozhukov_comparison_2015}]\label{lemma_chernozhukov_maximal}
Let $\bm{z}_1,\ldots,\bm{z}_n$ be independent, $d$-dimensional random vectors. Denote the $j$-th element of $\bm{z}_i$ by $z_{ij}$, $1\le j\le d$. Define $M:=\max_{1\le i\le n}\max_{1\le j\le d}|z_{ij}|$ and $\sigma^2:=\max_{1\le j\le d}\sum_{i=1}^n\EE[z_{ij}^2]$. Then,
\begin{align*}
    \EE\big[\max_{1\le j\le d}|\sum_{i=1}^n(z_{ij}-\EE[z_{ij}])|\big] \lesssim \sigma\sqrt{\log(d)} + \sqrt{\EE[M^2]}\log(d),
\end{align*}
where the universal constant in $\lesssim$ is positive and independent of $n$ and $d$.
\end{lemma}

\begin{lemma}[$L^q$ maximal inequality]\label{lemma_general_chernozhukov}
Let $\bm{x}_1,\ldots,\bm{x}_n$ be independent, $d$-dimensional random vectors. Denote by $x_{ij}$ the $j$-th element of $\bm{x}_i$, $1\le j\le d$. Then,
\begin{align*}
    \Big\| \max_{1\le j\le d}\Big|\sum_{i=1}^n\big(x_{ij}-\EE[x_{ij}]\big)\Big|\Big\|_q^2 \le e^2\big(\max\{q,\log(d)\}-1\big)\sum_{i=1}^n\Big\|\max_{1\le j\le d}\big|x_{ij}-\EE[x_{ij}]\big|\Big\|_q^2.
\end{align*}
\end{lemma}

This moment inequality can be derived by repeatedly applying Lemma~\ref{lemma_rio_Ls_moment_ineq}. It 
generalizes the maximal inequality for $\EE[\max_{1\le j\le d}|\sum_{i=1}^n(x_{ij}-\EE[x_{ij}])|]$ in \citet{chernozhukov_comparison_2015}, reproduced above as Lemma~\ref{lemma_chernozhukov_maximal}, to general $q$-th moments.

\begin{proof}[Proof of Lemma~\ref{lemma_general_chernozhukov}]
One can assume that the independent random vectors $\bm{x}_1,\ldots,\bm{x}_n$ have zero means. By repeatedly applying Lemma~\ref{lemma_rio_Ls_moment_ineq} and choosing $s=\log(d)$,
\begin{align*}
    \big\||\bm{x}_{1} + \cdots + \bm{x}_{n}|_{\infty}\big\|_{q}^{2} &\leq \big\||\bm{x}_{1} + \cdots + \bm{x}_{n}|_{s}\big\|_{q}^{2} \cr  
    &\leq \big\||\bm{x}_{1} + \cdots + \bm{x}_{n-1}|_{s}\big\|_{q}^{2} + (\max\{q, s\} - 1) \big\||\bm{x}_{n}|_{s}\big\|_{q}^{2} \cr 
    &\leq (\max\{q, s\} - 1) \sum_{i = 1}^{n} \big\||\bm{x}_{i}|_{s}\big\|_{q}^{2} \cr 
    &\leq e^{2} \big(\max\{q, \log(d)\} - 1\big) \sum_{i = 1}^{n} \big\||\bm{x}_{i}|_{\infty}\big\|_{q}^{2}.
\end{align*}
\end{proof}

\begin{lemma}[Moment inequality \citep{li_asymptotics_2024}]\label{lemma_moment_inequality}
    Let $q\ge2$. For any two random vectors $\bm{x}$ and $\bm{y}$ in $\RR^d$ with fixed $d\ge1$, and let
    $$\Delta = \EE\Big|\|\bm{x} + \bm{y}\|_2^q - \|\bm{x}\|_2^q - q\|\bm{x}\|_2^{q-2}\bm{x}^{\top}\bm{y}\Big|.$$
    Then, the following inequalities holds:
    
    \noindent(i) 
    $$\Delta \le \EE\big(\|\bm{x}\|_2 + \|\bm{y}\|_2\big)^q - \EE\|\bm{x}\|_2^q - q\EE(\|\bm{x}\|_2^{q-1}\|\bm{y}\|_2).$$

    \noindent(ii) 
    $$\Delta \le \big[(\EE\|\bm{x}\|_2^q)^{1/q} + (\EE\|\bm{y}\|_2^q)^{1/q}\big]^q - \EE\|\bm{x}\|_2^q - q(\EE\|\bm{x}\|_2^q)^{(q-1)/q}(\EE\|\bm{y}\|_2^q)^{1/q}.$$
\end{lemma}

\begin{lemma}[Equivalence of $\ell^s$-$\ell^{\infty}$-induced matrix norms]\label{lemma_Ls_L1_matrix_norm}
For matrix $A\in\RR^{d\times d}$, we have the equivalence of the $\ell^{s_d}$-norm and $\ell^{\infty}$-norm induced matrix norms as follows
\begin{equation}
    \label{eq_lemma_l_inf}
    \frac{1}{e}\|A\|_{\infty} \le \|A\|_{s_d} \le e\|A\|_{\infty},
\end{equation}
where $s_d$ is defined as~\eqref{eq_s_choice} and $\|A\|_s=\max_{|\bm{x}|_s\neq0}|A\bm{x}|_s/|\bm{x}|_s$. If in addition, $A$ is symmetric, then
\begin{equation}
    \label{eq_lemma_l_1}
    \frac{1}{e}\|A\|_1 =\frac{1}{e}\|A\|_{\infty} \le \|A\|_{s_d} \le e\|A\|_{\infty} = e\|A\|_1.
\end{equation}
\end{lemma}
\begin{proof}[Proof of Lemma~\ref{lemma_Ls_L1_matrix_norm}]
By \citet{horn_matrix_1985}, for any $1\le p\le q\le \infty$ and matrix $A\in\RR^{d\times d}$,
\begin{equation}
    d^{(1/q)-(1/p)}\|A\|_q \le \|A\|_p \le d^{(1/p)-(1/q)}\|A\|_q.
\end{equation}
For $p=s$ and $q=\infty$, we obtain
\begin{equation}
    d^{-1/s}\|A\|_{\infty}\le \|A\|_s\le d^{1/s}\|A\|_{\infty}.
\end{equation}
Since $d^{1/s}\le e$ by choosing $s=s_d$ in~\eqref{eq_s_choice}, we obtain~\eqref{eq_lemma_l_inf}. 

For symmetric $A=(a_{ij})_{1\le i,j\le d}$, $a_{ij}=a_{ji}$ for all $i,j$. Therefore,
\begin{equation}
    \|A\|_1 = \max_{1\le j\le d}\sum_{i=1}^d|a_{ij}| = \max_{1\le i\le d}\sum_{j=1}^d|a_{ij}| = \|A\|_{\infty}.
\end{equation}
This completes the proof.
\end{proof}

\subsection{Proofs for Section~\ref{sec_gmc}}

\textit{Derivation of \eqref{eq.9bdw8}:} Since $s$ is an even integer, we can write
\begin{equation}
    f(\alpha):=|\bm{x} - \alpha\bm{z}|_s^q = \Big\{\sum_{i=1}^d(x_i - \alpha z_i)^s\Big\}^{\frac{q}{s}}.
\end{equation}
Taking the derivative with respect to $\alpha$, we obtain
\begin{align}
    f'(\alpha) := \frac{d}{d\alpha}f(\alpha) & = \frac{q}{s}\Big\{\sum_{i=1}^d(x_i - \alpha z_i)^s\Big\}^{\frac{q}{s}-1}\sum_{i=1}^d\frac{d}{d\alpha}(x_i - \alpha z_i)^s \nonumber \\
    & = \frac{q}{s}\Big\{\sum_{i=1}^d(x_i - \alpha z_i)^s\Big\}^{\frac{q}{s}-1}\sum_{i=1}^ds(x_i - \alpha z_i)^{s-1}(-z_i) \nonumber \\
    & = -q\Big\{\sum_{i=1}^d(x_i - \alpha z_i)^s\Big\}^{\frac{q}{s}-1}\sum_{i=1}^d(x_i - \alpha z_i)^{s-1}z_i.
\end{align}
Therefore,
\begin{align}
    f'(0) & = -q\Big\{\sum_{i=1}^dx_i^s\Big\}^{\frac{q}{s}-1}\sum_{i=1}^dx_i^{s-1}z_i \nonumber \\
    & = -q|\bm{x}|_s^{q-s}\sum_{i=1}^dx_i^{s-1}z_i.
\end{align}
A first-order Taylor expansion yields then \eqref{eq.9bdw8}.

\begin{proof}[Proof of Lemma~\ref{lemma_Ls_moment_ineq}]
Recall that we have defined
\begin{equation*}
    f(\alpha)=|\bm{x} - \alpha\bm{z}|_s^q = \Big\{\sum_{i=1}^d(x_i - \alpha z_i)^s\Big\}^{\frac{q}{s}}.
\end{equation*}
A second order Taylor expansion gives $f(\alpha)=f(0) + \alpha f'(0)+\tfrac 12 \alpha^2 f''(\eta)$ for some $\eta\in [0,\alpha].$ It suffices to bound $\sup_{u\in[0,\alpha]}|f''(u)|$. Defining
\begin{equation}
    M(u) := \sum_{i=1}^d(x_i - u z_i)^s = |\bm{x} - u\bm{z}|_s^s,
\end{equation}
we have
$f(u)=[M(u)]^{\frac{q}{s}},$
\begin{align}
    M'(u) & = -s\sum_{i=1}^d(x_i-uz_i)^{s-1}z_i, \\
    M''(u) & = s(s-1)\sum_{i=1}^d(x_i-uz_i)^{s-2}z_i^2,
\end{align}
and the first two derivatives of $f(u)$ can be respectively expressed by
\begin{align}
    f'(u) & = \frac{q}{s}[M(u)]^{\frac{q}{s}-1}M'(u), \\
    f''(u) & = \frac{q}{s}\Big(\frac{q}{s} -1\Big)[M(u)]^{\frac{q}{s}-2}[M'(u)]^2 + M''(u)\frac{q}{s}[M(u)]^{\frac{q}{s}-1}.
\end{align}
Since $s$ is an even integer, it follows from Hölder's inequality that
\begin{align}
    [M'(u)]^2 & = s^2\Big(\sum_{i=1}^d(x_i-uz_i)^{s-1}z_i\Big)^2 \nonumber \\
    & \le s^2\bigg(\Big(\sum_{i=1}^d(x_i-uz_i)^s\Big)^{\frac{s-1}{s}}\Big(\sum_{i=1}^dz_i^s\Big)^{1/s}\bigg)^2 \nonumber \\
    & = s^2|\bm{x} -u\bm{z}|_s^{2(s-1)}|\bm{z}|_s^2.
\end{align}
By applying Hölder's inequality again, we obtain
\begin{align}
    \big|M''(u) \big| & \le s(s-1)\Big(\sum_{i=1}^d(x_i-uz_i)^s\Big)^{\frac{s-2}{s}}\Big(\sum_{i=1}^dz_i^s\Big)^{2/s} \nonumber \\
    & = s(s-1)|\bm{x} -u\bm{z}|_s^{s-2}|\bm{z}|_s^2.
\end{align}
By the two results above, we have
\begin{align}
    |f''(u)| & = \Bigg|\frac{q}{s}\Big(\frac{q}{s} -1\Big)|\bm{x} - u\bm{z}|_s^{q-2s}[M'(u)]^2 + M''(u)\frac{q}{s}|\bm{x} - u\bm{z}|_s^{q-s} \Bigg| \nonumber \\
    & \le q|q-s|\cdot|\bm{x} - u\bm{z}|_s^{q-2}|\bm{z}|_s^2 + q(s-1)|\bm{x} -u\bm{z}|_s^{q-2}|\bm{z}|_s^2 \nonumber \\
    & \le q\big[|q-s| + (s-1)\big]\big(|\bm{x}|_s + |u\bm{z}|_s\big)^{q-2}|\bm{z}|_s^2.
\end{align}
Since $u\in[0,\alpha]$, it follows that
\begin{equation}
    \sup_{u\in[0,\alpha]}|f''(u)| \le q\big[|q-s| + (s-1)\big]\big(|\bm{x}|_s + \alpha|\bm{z}|_s\big)^{q-2}|\bm{z}|_s^2.
\end{equation}
This completes the proof.
\end{proof}

\textit{Existence of solution to \texorpdfstring{\eqref{eq.b9fe}}{Eq.b9fe}: } To see the existence of the solution $\alpha_{s,q}$ in 
\[
1-q\mu\alpha + \frac{q\big[|q-s| + (s-1)\big]L_{s,q}^2}{2}\alpha^2 (1+\alpha L_{s,q})^{q-2} = 1.
\]
denote the function $\alpha \mapsto F(\alpha)=-\mu+c\alpha(1+L)^{q-2}$ for the constant $c=[|q-s| + (s-1)]L^2/2 >0$ and $L=L_{s,q}$. For any $q\ge2$, and any $\alpha>0,$ $F'(\alpha)=c[(1+L\alpha)^{q-2}+\alpha(q-2)L(1+L\alpha)^{q-3}]>0,$ proving that $F(\alpha)$ is strictly increasing on $\alpha>0$. Since $F(0)=-\mu<0$ and $F(\infty)=+\infty$, the unique root to $F(\alpha)=0$ exists.

\begin{proof}[Proof of Lemma~\ref{lemma_rio_Ls_moment_ineq}]
Define $\varphi(t) = \||\bm{x} + t \bm{y}|_{s}\|_{q}^{2}$ for $t \in [0, 1]$. Then
\begin{align*}
    \varphi'(t) &= \frac{2}{q} \Big[\mathbb{E} \Big\{\sum_{j = 1}^{d} (x_{j} + t y_{j})^{s}\Big\}^{q/s}\Big]^{2/q - 1}  \frac{q}{s} \mathbb{E} \Big[\Big\{\sum_{j = 1}^{d} (x_{j} + t y_{j})^{s}\Big\}^{q/s - 1} \sum_{j = 1}^{d} s (x_{j} + t y_{j})^{s - 1} y_{j}\Big] \cr 
    &= 2 \Big[\mathbb{E} \Big\{\sum_{j = 1}^{d} (x_{j} + t y_{j})^{s}\Big\}^{q/s}\Big]^{2/q - 1} \mathbb{E} \Big[\Big\{\sum_{j = 1}^{d} (x_{j} + t y_{j})^{s}\Big\}^{q/s - 1} \sum_{j = 1}^{d} (x_{j} + t y_{j})^{s - 1} y_{j}\Big] 
\end{align*}
and 
\begin{align*}
    \varphi''(t) &= 2 \Big(\frac 2q-1\Big) \Big[\mathbb{E} \Big\{\sum_{j = 1}^{d} (x_{j} + t y_{j})^{s}\Big\}^{q/s}\Big]^{2/q - 2} \\
    & \quad \cdot \frac{q}{s} \cdot s \Big|\mathbb{E} \Big[\Big\{\sum_{j = 1}^{d} (x_{j} + t y_{j})^{s}\Big\}^{q/s - 1} \sum_{j = 1}^{d} (x_{j} + t y_{j})^{s - 1} y_{j}\Big]\Big|^{2} \cr 
    &+ 2 \Big[\mathbb{E} \Big\{\sum_{j = 1}^{d} (x_{j} + t y_{j})^{s}\Big\}^{q/s}\Big]^{2/q - 1} \\
    & \quad \cdot \mathbb{E} \Big[\Big(\frac qs-1\Big) \Big\{\sum_{j = 1}^{d} (x_{j} + t y_{j})^{s}\Big\}^{q/s - 2} s \Big\{\sum_{j = 1}^{d} (x_{j} + t y_{j})^{s - 1} y_{j}\Big\}^{2}\Big] \cr 
    &+ 2 \Big[\mathbb{E} \Big\{\sum_{j = 1}^{d} (x_{j} + t y_{j})^{s}\Big\}^{q/s}\Big]^{2/q - 1} \mathbb{E} \Big[\Big\{\sum_{j = 1}^{d} (x_{j} + t y_{j})^{s}\Big\}^{q/s - 1} (s - 1) \sum_{j = 1}^{d} (x_{j} + t y_{j})^{s - 2} y_{j}^{2}\Big] \cr 
    &=: \Delta_{1} (t) + \Delta_{2} (t) + \Delta_{3} (t)
\end{align*}
Since $q \geq 2$, $\Delta_{1}(t) \leq 0$.

\textbf{Case I.} If $q/s - 1 \leq 0$, then $\Delta_{2} (t) \leq 0$ and $\varphi''(t) \leq \Delta_{3} (t)$. By H\"older's inequality, 
    \begin{align*}
        \sum_{j = 1}^{d} (x_{j} + t y_{j})^{s - 2} y_{j}^{2} \leq \Big\{\sum_{j = 1}^{d} (x_{j} + t y_{j})^{s}\Big\}^{(s - 2)/s} \Big(\sum_{j = 1}^{d} y_{j}^{s}\Big)^{2/s} 
    \end{align*}
    Consequently, 
    \begin{align*}
        \Delta_{3} (t) &\leq 2 (s - 1) \Big[\mathbb{E} \Big\{\sum_{j = 1}^{d} (x_{j} + t y_{j})^{s}\Big\}^{q/s}\Big]^{2/q - 1} \\
        & \quad \cdot\mathbb{E} \Big[\Big\{\sum_{j = 1}^{d} (x_{j} + t y_{j})^{s}\Big\}^{q/s - 1} \Big\{\sum_{j = 1}^{d} (x_{j} + t y_{j})^{s}\Big\}^{(s - 2)/s} \Big(\sum_{j = 1}^{d} y_{j}^{s}\Big)^{2/s}\Big] \cr 
        &= 2 (s - 1) \Big[\mathbb{E} \Big\{\sum_{j = 1}^{d} (x_{j} + t y_{j})^{s}\Big\}^{q/s}\Big]^{2/q - 1} \mathbb{E} \Big[\Big\{\sum_{j = 1}^{d} (x_{j} + t y_{j})^{s}\Big\}^{(q - 2)/s} \Big(\sum_{j = 1}^{d} y_{j}^{s}\Big)^{2/s}\Big] \cr 
        &= 2 (s - 1) \||\bm{x} + t \bm{y}|_{s}\|_{q}^{2 - q} \mathbb{E} \Big[\Big\{\sum_{j = 1}^{d} (x_{j} + t y_{j})^{s}\Big\}^{(q - 2)/s} \Big(\sum_{j = 1}^{d} y_{j}^{s}\Big)^{2/s}\Big] \cr 
        &\leq 2 (s - 1) \||\bm{x} + t \bm{y}|_{s}\|_{q}^{2 - q} \||\bm{x} + t \bm{y}|_{s}\|_{q}^{q - 2} \||\bm{y}|_{s}\|_{q}^{2} \cr 
        &= 2 (s - 1) \||\bm{y}|_{s}\|_{q}^{2}. 
    \end{align*}
    
\textbf{Case II.} If $q/s - 1 > 0$, by H\"older's inequality, 
    \begin{align*}
        \Big\{\sum_{j = 1}^{d} (x_{j} + t y_{j})^{s - 1} y_{j}\Big\}^{2} 
        &= \Big\{\sum_{j = 1}^{d} (x_{j} + t y_{j})^{s/2} (x_{j} + t y_{j})^{s/2 - 1} y_{j}\Big\}^{2} \cr 
        &\leq \sum_{j = 1}^{d} (x_{j} + t y_{j})^{s} \sum_{j = 1}^{d} (x_{j} + t y_{j})^{s - 2} y_{j}^{2}. 
    \end{align*}
    Therefore, 
    \begin{align*}
        \Delta_{2} (t) \leq \Delta_{3} (t)  \, \frac{q - s}{s - 1}
    \end{align*}
    and
    \begin{align*}
        \varphi''(t) \leq \Delta_{2} (t) + \Delta_{3} (t) \leq \Delta_{3} (t) \, \frac{q - 1}{s - 1} \leq 2 (q - 1) \||\bm{y}|_{s}\|_{q}^{2}. 
    \end{align*}
    Then, we have 
    \begin{align*}
        \big\||\bm{x} + \bm{y}|_{s}\big\|_{q}^{2} = \varphi (1) &= \varphi (0) + \varphi'(0) + \int_{0}^{1} (1 - t) \varphi''(t) \, dt \cr 
        &\leq \big\||\bm{x}|_{s}\big\|_{q}^{2} + 2 \big\||\bm{x}|_{s}\big\|_{q}^{2 - q} \mathbb{E} \Big(|\bm{x}|_{s}^{q - s} \sum_{j = 1}^{d} x_{j}^{s - 1} y_{j}\Big) + \big(\max\{q, s\} - 1\big) \big\||\bm{y}|_{s}\big\|_{q}^{2}. 
    \end{align*}

\end{proof}

\begin{proof}[Proof of Theorem~\ref{thm_gmc}]
Consider the iterated random function
\begin{align}
    F:\RR^d\times \RR \mapsto \RR, \quad (\bbeta,\bm{\xi}) \mapsto F_{\bm{\xi}}(\bbeta) = \bbeta - \alpha\nabla g(\bbeta,\bm{\xi}).
\end{align}
To prove GMC in Theorem~\ref{thm_gmc}, it suffices to show that, for some $q\ge2$ and even integer $s\ge2$, for any fixed vectors $\bbeta,\bbeta'\in\RR^d$,
\begin{align*}
    \||F_{\bm{\xi}}(\bbeta) - F_{\bm{\xi}}(\bbeta')|_s\|_q \le r_{\alpha,s,q}|\bbeta-\bbeta'|_s.
\end{align*}
Recall the inequality in Lemma~\ref{lemma_rio_Ls_moment_ineq}. For $\bm{x}$ and $\bm{y}$ therein, we choose them to be $\bm{x}=\bbeta-\bbeta'$ and $\bm{y}= - \alpha(\nabla g(\bbeta,\bm{\xi}) - \nabla g(\bbeta',\bm{\xi}))$ respectively. Then, it directly follows from Lemma~\ref{lemma_rio_Ls_moment_ineq} that
\begin{align*}
    & \quad \||F_{\bm{\xi}}(\bbeta) - F_{\bm{\xi}}(\bbeta')|_s\|_q^2 \\
    & \le |\bbeta-\bbeta'|_s^2 - 2\alpha|\bbeta-\bbeta'|_s^{2-q}\EE\Big[|\bbeta-\bbeta'|_s^{q-s}\big\langle (\bbeta-\bbeta')^{s-1},\nabla g(\bbeta,\bm{\xi}) - \nabla g(\bbeta',\bm{\xi})\big\rangle\Big] \\
    & \quad +\alpha^2(\max\{q,s\}-1)\||\nabla g(\bbeta,\bm{\xi}) - \nabla g(\bbeta',\bm{\xi})|_s\|_q^2 \\
    & =  |\bbeta-\bbeta'|_s^2 - 2\alpha|\bbeta-\bbeta'|_s^{2-s}\big\langle (\bbeta-\bbeta')^{s-1},G(\bbeta) - G(\bbeta')\big\rangle \\
    & \quad +\alpha^2(\max\{q,s\}-1)\||\nabla g(\bbeta,\bm{\xi}) - \nabla g(\bbeta',\bm{\xi})|_s\|_q^2.
\end{align*}
This along with Assumptions~\ref{asm_Ls_strong_convex}~and~\ref{asm_Ls_lip} yields 
\begin{align*}
    \||F_{\bm{\xi}}(\bbeta) - F_{\bm{\xi}}(\bbeta')|_s\|_q^2 \le \big(1-2\alpha\mu+\alpha^2(\max\{q,s\}-1)L_{s,q}^2\big)|\bbeta-\bbeta'|_s^2, 
\end{align*}
which completes the proof.
\end{proof}

\subsection{Proofs for Section~\ref{subsec_moment_sgd}}

\begin{proof}[Proof of Proposition~\ref{prop_rio_sgd}]
Recall~\eqref{eq_partials_G} and let $\nabla G(\bbeta) = \big(\nabla G_1(\bbeta),\ldots,\nabla G_d(\bbeta)\big)^{\top}$ with 
\begin{equation}
    \label{eq_m_beta}
    \nabla G_i(\bbeta) = \partial G(\bbeta)/\partial \beta_i= \big(\EE[\nabla g(\bbeta,\bm{\xi})]\big)_i, \quad i=1,\ldots,d.
\end{equation}
Since the random samples $\bm{\xi}_k$, $k\ge1$, are independent, it follows that for the $k$-th iteration, $\bm{\xi}_k$ is independent of $\bbeta_{k-1}$. Then, by the tower rule, for all $k\ge1$,
\begin{equation}
    \EE_{\bm{\xi}}\big[\nabla g(\bbeta_{k-1},\bm{\xi}_k) - \nabla G(\bbeta_{k-1}) \mid \bbeta_{k-1}\big] = \EE_{\bm{\xi}}[\nabla g(\bbeta_{k-1},\bm{\xi}_k) - \nabla G(\bbeta_{k-1})] = 0.
\end{equation}
Therefore, by applying the high-dimensional moment inequality~\eqref{eq_hd_rio} in Lemma~\ref{lemma_rio_Ls_moment_ineq}, we obtain
\begin{align}
    \label{eq_rio_sgd_goal}
    \big\||\bbeta_k-\bbeta^*|_s\big\|_q^2 & \le \big\||\bbeta_{k-1}-\bbeta^* - \alpha \nabla G(\bbeta_{k-1})|_s\big\|_q^2 \nonumber \\
    & \quad + (\max\{q,s\}-1)\alpha ^2\big\||\nabla g(\bbeta_{k-1},\bm{\xi}_k)-\nabla G(\bbeta_{k-1})|_s\big\|_q^2. 
\end{align}
For the second part in~\eqref{eq_rio_sgd_goal}, noting that $\nabla G(\bbeta^*)=0$, by the triangle inequality, we have
\begin{align}
    \label{eq_rio_goal_part2}
    & \quad \big\||\nabla g(\bbeta_{k-1},\bm{\xi}_k)-\nabla G(\bbeta_{k-1})|_s\big\|_q^2 \nonumber \\
    & \le \Big(\big\||\nabla g(\bbeta_{k-1},\bm{\xi}_k)-\nabla g(\bbeta^*,\bm{\xi}_k)|_s\big\|_q + \big\||\nabla G(\bbeta_{k-1}) - \nabla G(\bbeta^*)|_s\big\|_q + \big\||\nabla g(\bbeta^*,\bm{\xi}_k)|_s\big\|_q\Big)^2 \nonumber \\
    & \le 3\big\||\nabla g(\bbeta_{k-1},\bm{\xi}_k)-\nabla g(\bbeta^*,\bm{\xi}_k)|_s\big\|_q^2 + 3\big\||\nabla G(\bbeta_{k-1}) - \nabla G(\bbeta^*)|_s\big\|_q^2 + 3\big\||\nabla g(\bbeta^*,\bm{\xi}_k)|_s\big\|_q^2.
\end{align}
Since $|\cdot|_s$ is a convex function for $s\ge1$, we have $|\EE[\cdot]|_s \le \EE[|\cdot|_s].$ Thus, for all $q\ge1$, by Jensen's inequality, we can bound
\begin{align}
    |\nabla G(\bbeta_{k-1}) - \nabla G(\bbeta^*)|_s & = \big|\EE_{\bm{\xi}}\big[\nabla g(\bbeta_{k-1},\bm{\xi}_k)-\nabla g(\bbeta^*,\bm{\xi}_k)\big]\big|_s \nonumber \\
    & \le \EE_{\bm{\xi}}\Big[\big|\nabla g(\bbeta_{k-1},\bm{\xi}_k)-\nabla g(\bbeta^*,\bm{\xi}_k)\big|_s\Big] \nonumber \\
    & \le \Big(\EE_{\bm{\xi}}\big|\nabla g(\bbeta_{k-1},\bm{\xi}_k)-\nabla g(\bbeta^*,\bm{\xi}_k)\big|_s^q\Big)^{1/q}.
\end{align}
This along with Assumption~\ref{asm_Ls_lip} yields 
\begin{align}
    \big\||\nabla G(\bbeta_{k-1}) - \nabla G(\bbeta^*)|_s\big\|_q & \le   \Big(\EE_{\bbeta}\EE_{\bm{\xi}}\big|\nabla g(\bbeta_{k-1},\bm{\xi}_k)-\nabla g(\bbeta^*,\bm{\xi}_k)\big|_s^q\Big)^{1/q}  \nonumber \\
    & =  \Big\|\big|\nabla g(\bbeta_{k-1},\bm{\xi}_k)-\nabla g(\bbeta^*,\bm{\xi}_k)\big|_s\Big\|_q  \nonumber \\
    & \le L_{s,q}\big\||\bbeta_{k-1}-\bbeta^*|_s\big\|_q.
\end{align}
Inserting this result back into~\eqref{eq_rio_goal_part2}, we obtain a bound for the second term in~\eqref{eq_rio_sgd_goal} using
\begin{align}
    \label{eq_rio_goal_part2_bound}
    \big\||\nabla g(\bbeta_{k-1},\bm{\xi}_k)-\nabla G(\bbeta_{k-1})|_s\big\|_q^2 \le 6L_{s,q}^2\big\||\bbeta_{k-1}-\bbeta^*|_s\big\|_q^2 + 3\big\||\nabla g(\bbeta^*,\bm{\xi}_k)|_s\big\|_q^2.
\end{align}
For the first term in~\eqref{eq_rio_sgd_goal}, by applying Lemma~\ref{lemma_rio_Ls_moment_ineq} again, it follows from Assumptions~\ref{asm_Ls_strong_convex}~and~\ref{asm_Ls_lip} that
\begin{align}
    & \quad \big\||\bbeta_{k-1}-\bbeta^* - \alpha \nabla G(\bbeta_{k-1})|_s\big\|_q^2 \nonumber \\
    & \le \big\||\bbeta_{k-1}-\bbeta^*|_s\big\|_q^2 - 2\alpha\big\||\bbeta_{k-1}-\bbeta^*|_s\big\|_q^{2-q}\EE\Big(|\bbeta_{k-1}-\bbeta^*|_{s}^{q - s} \sum_{j = 1}^{d} (\bbeta_{k-1}-\bbeta^*)_j^{s - 1} \nabla G_j(\bbeta_{k-1})\Big) \nonumber \\
    & \quad + \alpha^2(\max\{q,s\}-1)\big\||\nabla G(\bbeta_{k-1}) - \nabla G(\bbeta^*)|_s\big\|_q^2 \nonumber \\
    & \le \big(1-2\alpha\mu+\alpha^2(\max\{q,s\}-1)L_{s,q}^2\big)\big\||\bbeta_{k-1}-\bbeta^*|_s\big\|_q^2.
\end{align}
Inserting this inequality and~\eqref{eq_rio_goal_part2_bound} into~\eqref{eq_rio_sgd_goal}, we obtain the inequality
\begin{align*}
    \||\bbeta_k-\bbeta^*|_s\|_q^2 & \le \big(1-2\alpha\mu+7(\max\{q,s\}-1)\alpha^2L_{s,q}^2\big)\||\bbeta_{k-1}-\bbeta^*|_s\|_q^2 \nonumber \\
    & \quad + 3(\max\{q,s\}-1)\alpha^2\||\nabla g(\bbeta^*,\bm{\xi}_k)|_s\|_q^2.
\end{align*}
The desired result is achieved since $\||\nabla g(\bbeta^*,\bm{\xi}_k)|_s\|_q\le M_{s,q}$ by Assumption~\ref{asm_Ls_lip}. As a special case, for the stationary SGD iterates $\bbeta_k^{\circ}\sim\pi_{\alpha}$, $k\ge1$, we obtain the same result.
\end{proof}

\begin{proof}[Proof of Theorem~\ref{thm_sgd_moment}]
First, we denote the contraction constant in Proposition~\ref{prop_rio_sgd} as follows
\begin{equation}
    \label{eq_tilde_r}
    \tilde r_{\alpha,s,q} := 1-2\alpha\mu+7(\max\{q,s\}-1)\alpha^2L_{s,q}^2.
\end{equation}
Given the range of the constant learning rate $\alpha$, we have $\tilde r_{\alpha,s,q}<1$. Moreover, notice that
\begin{equation}
    \label{eq_gradient_beta_star}
    3(\max\{q,s\}-1)\alpha^2\||\nabla g(\bbeta^*,\bm{\xi}_k)|_s\|_q^2 = O\big(\max\{q,s\}\alpha^2M_{s,q}^2\big).
\end{equation}
Therefore, for the stationary SGD iterates $\bbeta_k^{\circ}\sim\pi_{\alpha}$, by Proposition~\ref{prop_rio_sgd}, we can obtain
\begin{equation}
    \||\bbeta_k^{\circ}-\bbeta^*|_s\|_q^2 \le \tilde r_{\alpha,s,q}\||\bbeta_{k-1}^{\circ}-\bbeta^*|_s\|_q^2 + O\big(\max\{q,s\}\alpha^2M_{s,q}^2\big).
\end{equation}
Since the SGD iterates $\bbeta_k^{\circ}$ satisfy the geometric-moment contraction in Theorem~\ref{thm_gmc}, following Remark 2 in \citet{wu_limit_2004}, the recursion $\bbeta_k^{\circ}=\bbeta_{k-1}^{\circ}-\alpha\nabla g(\bbeta_{k-1}^{\circ},\bm{\xi}_k)$ also holds for $k\le0$. Thus, we can recursively apply the inequality above and achieve
\begin{align}
    \||\bbeta_k^{\circ}-\bbeta^*|_s\|_q^2 & \le O\big(\max\{q,s\}\alpha^2M_{s,q}^2\big) \cdot \sum_{i=0}^{\infty}\tilde r_{\alpha,s,q}^i \nonumber \\
    & = \frac{1}{1-\tilde r_{\alpha,s,q}}O\big(\max\{q,s\}\alpha^2M_{s,q}^2\big) \nonumber \\
    & = O\big(\max\{q,s\}\alpha M_{s,q}^2\big).
\end{align}
This finishes the proof for the stationary SGD sequence.

Furthermore, for the general SGD iterates $\bbeta_k$ in~\eqref{eq_sgd_recursion} that may not have the stationary initialization, we apply the geometric-moment contraction in Theorem~\ref{thm_gmc} and obtain
\begin{align}
    \||\bbeta_k-\bbeta^*|_s\|_q & \le \||\bbeta_k - \bbeta_k^{\circ}|_s\|_q + \||\bbeta_k^{\circ}-\bbeta^*|_s\|_q \nonumber \\
    & \le r_{\alpha,s,q}^k\||\bbeta_0-\bbeta_0^{\circ}|_s\|_q + O\big(M_{s,q}\sqrt{\max\{q,s\}\alpha}\big),
\end{align}
which completes the proof.
\end{proof}

\subsection{Functional Dependence Measure in Time Series}\label{subsubsec_functional_dep}

The functional dependence measure in time series \citep{wu_nonlinear_2005} is a key concept in our analysis. For that we view  the high-dimensional SGD iterates $\{\bbeta_k\}_{k\in\NN}$ as a nonlinear autoregressive (AR) process. Recall that $\bm{\xi}_k$, $k\in \ZZ$, are i.i.d. Define the shift process $\F_k=(\bm{\xi}_k,\bm{\xi}_{k-1},\ldots)$ and its coupled version $\F_{k,\{l\}}=(\bm{\xi}_k,\ldots,\bm{\xi}_{l+1},\bm{\xi}_l',\bm{\xi}_{l-1},\ldots)$, $l\le k$, where $\bm{\xi}_l'$ is an i.i.d. copy of $\bm{\xi}_l$. 

The stationary sequence $\{\bbeta_k^{\circ}\}_{k\in\ZZ}$ can be represented by a functional system
\begin{equation}
    \label{eq_beta_stationary_repre}
    \bbeta_k^{\circ}=h_{\alpha}(\bm{\xi}_k,\bm{\xi}_{k-1},\ldots)=h_{\alpha}(\F_k), \quad k\ge1,
\end{equation}
where $h_{\alpha}$ is a measurable function that depends on $\alpha$ \citep{wiener_nonlinear_1958,wu_nonlinear_2005}. Define the coupled version of $\bbeta_k^{\circ}$ by
\begin{equation}
    \label{eq_functional_dep_def}
    \bbeta_{k,\{l\}}^{\circ} = h_{\alpha}(\bm{\xi}_k,\ldots,\bm{\xi}_{l+1},\bm{\xi}_l',\bm{\xi}_{l-1},\ldots)=h_{\alpha}(\F_{k,\{l\}}), \quad l\le k.
\end{equation}
The next lemma provides a bound for the functional dependence measure $\| |\bbeta_k^{\circ}- \bbeta_{k,\{l\}}^{\circ}|_s \|_q.$ It is later used to derive the moment bounds and the tail probability of the ASGD iterates.

\begin{lemma}
\label{lemma_functional_dep}
Consider the stationary SGD sequence $\{\bbeta_k^{\circ}\}_{k\ge1}$. Suppose that Assumptions~\ref{asm_Ls_strong_convex}~and~\ref{asm_Ls_lip} hold with some $q\ge2$ and even integer $s\ge2$. Then, for all $k\ge1$ and $l\le k$, we have
\begin{equation*}
    \| |\bbeta_k^{\circ}- \bbeta_{k,\{l\}}^{\circ}|_s \|_q^2 \le 4\alpha^2\big(1-2\alpha\mu+7(\max\{q,s\}-1)\alpha^2L_{s,q}^2\big)^{k-l}\Big(L_{s,q}^2\||\bbeta_{l-1}^{\circ}-\bbeta^*|_s\|_q^2 + M_{s,q}^2\Big).
\end{equation*}
\end{lemma}

\begin{proof}[Proof of Lemma~\ref{lemma_functional_dep}]
By applying Lemma~\ref{lemma_rio_Ls_moment_ineq}, it follows from similar arguments as in the proof of Proposition~\ref{prop_rio_sgd} that, for each $l\le k-1$,
\begin{align}
    \| |\bbeta_k^{\circ}- \bbeta_{k,\{l\}}^{\circ}|_s \|_q^2 
    & \le \big(1-2\alpha\mu+7(\max\{q,s\}-1)\alpha^2L_{s,q}^2\big)^{k-l}\| |\bbeta_l^{\circ}- \bbeta_{l,\{l\}}^{\circ}|_s \|_q^2.
\end{align}
By Assumption~\ref{asm_Ls_lip}, for all $l\ge1$,
\begin{align}
    \||\nabla g(\bbeta_{l-1}^{\circ},\bm{\xi}_l)|_s\|_q^2 & \le 2\||\nabla g(\bbeta_{l-1}^{\circ},\bm{\xi}_l) - \nabla g(\bbeta^*,\bm{\xi}_l)|_s\|_q^2 + 2\|\nabla g(\bbeta^*,\bm{\xi}_l)|_s\|_q^2 \nonumber \\
    & \le 2L_{s,q}^2\||\bbeta_{l-1}^{\circ}-\bbeta^*|_s\|_q^2 + 2M_{s,q}^2,
\end{align}
which yields 
\begin{align}
    \| |\bbeta_l^{\circ}- \bbeta_{l,\{l\}}^{\circ}|_s \|_q^2 & = \alpha^2\||\nabla g(\bbeta_{l-1}^{\circ},\bm{\xi}_l) - \nabla g(\bbeta_{l-1}^{\circ},\bm{\xi}_l')|_s\|_q^2 \nonumber \\
    & \le \alpha^2\Big(2\||\nabla g(\bbeta_{l-1}^{\circ},\bm{\xi}_l)|_s\|_q^2 + 2\||\nabla g(\bbeta_{l-1}^{\circ},\bm{\xi}_l')|_s\|_q^2\Big) \nonumber \\
    & \le 4\alpha^2\Big(L_{s,q}^2\||\bbeta_{l-1}^{\circ}-\bbeta^*|_s\|_q^2 + M_{s,q}^2\Big)
\end{align}
Recall $\||\nabla g(\bbeta^*,\bm{\xi}_k)|_s\|_q\le M_{s,q}$ by Assumption~\ref{asm_Ls_lip}. Therefore,
\begin{align}
    \| |\bbeta_k^{\circ}- \bbeta_{k,\{l\}}^{\circ}|_s \|_q^2 & \le 4\alpha^2\big(1-2\alpha\mu+7(\max\{q,s\}-1)\alpha^2L_{s,q}^2\big)^{k-l} \nonumber \\
    & \quad \cdot \Big(L_{s,q}^2\||\bbeta_{l-1}^{\circ}-\bbeta^*|_s\|_q^2 + M_{s,q}^2\Big).
\end{align}
This completes the proof.
\end{proof}

\subsection{Proofs for Section~\ref{subsec_moment_asgd}}

In this section, we provide the proofs for the convergence results of ASGD in Section~\ref{subsec_moment_asgd}, which can be decomposed into the proofs for Theorems~\ref{thm_asgd_stationary}~to~\ref{thm_bias} in Section~\ref{sec_proof_sketch}.

\begin{proof}[Proof of Theorem~\ref{thm_asgd_stochastic}]
Recall the i.i.d.\ random samples $\bm{\xi}_k=(y_k,\bm{x}_k)$, the filtration $\F_k=(\bm{\xi}_k,\bm{\xi}_{k-1},\ldots)$ and its coupled version $\F_{k,\{l\}}=(\bm{\xi}_k,\ldots,\bm{\xi}_{l+1},\bm{\xi}_l',\bm{\xi}_{l-1},\ldots)$, $l\le k$, where $\bm{\xi}_l'$ is an i.i.d. copy of $\bm{\xi}_l$. 
Following \citet{wu_nonlinear_2005}, we introduce the projection operator 
$$\P_l[\cdot] = \EE[\cdot\mid\F_l] - \EE[\cdot \mid\F_{l-1}].$$ Then, we can rewrite the centered ASGD into
\begin{align}
    \bar\bbeta_k^{\circ} - \EE[\bar\bbeta_k^{\circ}] = \frac{1}{k}\sum_{i=1}^k\sum_{l=0}^{i-1}\P_{i-l}(\bbeta_i^{\circ}) = \frac{1}{k}\sum_{l=0}^{k-1}\sum_{i=l+1}^k\P_{i-l}(\bbeta_i^{\circ}).
\end{align}
Since $\{\P_{i-l}(\bbeta_i^{\circ})\}_{i\ge l+1}$ is a sequence of martingale differences over $i$ for each $l=0,1,\ldots,i-1$, following Lemma D.2 in \citet{zhang_convergence_2021} and triangle inequality, we can obtain
\begin{align}
    \label{eq_agsd_step0}
    \||\bar\bbeta_k^{\circ} - \EE[\bar\bbeta_k^{\circ}]|_s\|_q & = \Big\| \Big|\frac{1}{k}\sum_{l=0}^{k-1}\sum_{i=l+1}^k\P_{i-l}(\bbeta_i^{\circ})\Big|_s\Big\|_q \nonumber \\
    & \le \frac{1}{k}\sum_{l=0}^{k-1}\Big\| \Big|\sum_{i=l+1}^k\P_{i-l}(\bbeta_i^{\circ})\Big|_s\Big\|_q \nonumber \\
    & \le \frac{1}{k}\sum_{l=0}^{k-1}\Big(c_q\cdot s\sum_{i=l+1}^k\big\||\P_{i-l}(\bbeta_i^{\circ})|_s\big\|_q^2\Big)^{1/2}.
\end{align}
By Theorem 1 in \citet{wu_nonlinear_2005}, we have
\begin{align}
    \||\P_{i-l}(\bbeta_i^{\circ})|_s\|_q \le \||\bbeta_i^{\circ} - \bbeta_{i,\{i-l\}}^{\circ}|_s\|_q.
\end{align}
This along with Lemma~\ref{lemma_functional_dep} and definition of $\tilde r_{\alpha,s,q}$ in~\eqref{eq_gradient_beta_star} yields
\begin{align}
    \||\P_{i-l}(\bbeta_i^{\circ})|_s\|_q^2 & \le 4\alpha^2\big(1-2\alpha\mu+7(\max\{q,s\}-1)\alpha^2L_{s,q}^2\big)^l \nonumber \\
    & \quad \cdot \Big(L_{s,q}^2\||\bbeta_{i-l-1}^{\circ}-\bbeta^*|_s\|_q^2 + M_{s,q}^2\Big) \nonumber \\
    & = 4\alpha^2\tilde r_{\alpha,s,q}^l\Big(L_{s,q}^2\||\bbeta_{i-l-1}^{\circ}-\bbeta^*|_s\|_q^2 + M_{s,q}^2\Big).
\end{align}
Recall $r_{\alpha,s,q}$ in~\eqref{eq_gmc} and $\tilde r_{\alpha,s,q}$ in~\eqref{eq_gradient_beta_star}. For some constant $\omega>0$ such that
\begin{equation}
    \label{eq_omega}
    \omega \le \min\Big\{\frac{1}{\alpha},2\mu-7(\max\{q,s\}-1)\alpha L_{s,q}^2\Big\},
\end{equation}
we have $1-\omega\alpha\ge0$ and
\begin{equation}
    r_{\alpha,s,q} \le \tilde r_{\alpha,s,q} \le 1-\omega\alpha <1.
\end{equation}
Consequently, we can further bound~\eqref{eq_agsd_step0} by
\begin{align}
    \label{eq_asgd_step1}
    & \quad \||\bar\bbeta_k^{\circ} - \EE[\bar\bbeta_k^{\circ}]|_s\|_q \nonumber \\
    & \le \frac{1}{k}\sum_{l=0}^{k-1}\Bigg[4c_qs\alpha^2(1-\omega\alpha)^l\sum_{i=l+1}^k\Big(L_{s,q}^2\||\bbeta_{i-l-1}^{\circ}-\bbeta^*|_s\|_q^2 + M_{s,q}^2\Big)\Bigg]^{1/2} \nonumber \\
    & = \frac{1}{k}\sum_{l=0}^{k-1}\Bigg[4c_qs\alpha^2(1-\omega\alpha)^l\Big(L_{s,q}^2\sum_{i=l+1}^k\||\bbeta_{i-l-1}^{\circ}-\bbeta^*|_s\|_q^2 + (k-l)M_{s,q}^2\Big)\Bigg]^{1/2} \nonumber \\
    & \le \frac{1}{k}\sum_{l=0}^{k-1}\Bigg[2\alpha\sqrt{c_qs}(1-\omega\alpha)^{l/2}L_{s,q}\sqrt{\sum_{i=l+1}^k\||\bbeta_{i-l-1}^{\circ}-\bbeta^*|_s\|_q^2}\Bigg] \nonumber \\
    & \quad + \frac{1}{k}\sum_{l=0}^{k-1}\Bigg[2\alpha\sqrt{c_qs}(1-\omega\alpha)^{l/2}\sqrt{(k-l)}M_{s,q}\Bigg] =: \III_1+\III_2.
\end{align}
For the term $\III_1$, it follows from Theorem~\ref{thm_sgd_moment} and expression~\eqref{eq_gradient_beta_star} that
\begin{align}
    \sum_{i=l+1}^k\||\bbeta_{i-l-1}^{\circ}-\bbeta^*|_s\|_q^2 & \le \sum_{i=l+1}^k\Big(6M_{s,q}^2(\max\{q,s\}-1)\alpha\Big) \nonumber \\
    & = 6\alpha (k-l)(\max\{q,s\}-1)M_{s,q}^2.
\end{align}
Inserting this back into~\eqref{eq_asgd_step1} gives
\begin{align}
    \III_1 & \le \frac{2\alpha\sqrt{c_qs}L_{s,q}}{k}\sum_{l=0}^{k-1}(1-\omega\alpha)^{l/2}M_{s,q}\sqrt{6\alpha (k-l)(\max\{q,s\}-1)}\nonumber \\
    & \le \sqrt{c_qs}L_{s,q}\cdot \frac{c_1\sqrt{\alpha}}{\sqrt{k}}M_{s,q}\sqrt{\max\{q,s\}-1},
\end{align}
for some constant $c_1>0$, where the last inequality is due to
\begin{align}
    \sum_{l=0}^{k-1}(1-\omega\alpha)^{l/2}\sqrt{k-l} \le \sqrt{k}\sum_{l=0}^{k-1}(1-\omega\alpha)^{l/2}= O\Big(\frac{\sqrt{k}}{\omega\alpha}\Big).
\end{align}
Similarly, for some constant $c_2>0$,
\begin{equation}
    \III_2 \le \frac{c_2\sqrt{c_qs}}{\sqrt{k}}M_{s,q}.
\end{equation}
Combining the results of $\III_1$ and $\III_2$, we obtain the claimed inequality
\begin{align*}
    \||\bar\bbeta_k^{\circ} - \EE[\bar\bbeta_k^{\circ}]|_s\|_q \le \sqrt{\frac{c_qs}{k}}\Big(c_1L_{s,q}\sqrt{\alpha}M_{s,q}\sqrt{\max\{q,s\}-1} + c_2M_{s,q}\Big).
\end{align*}
\end{proof}

\begin{proof}[Proof of Theorem~\ref{thm_asgd_stationary}]
For the ASGD sequence $\{\bar\bbeta_k\}_{k\in\NN}$ with arbitrarily fixed initialization $\bbeta_0\in\RR^d$ and the stationary ASGD sequence $\{\bar\bbeta_k^{\circ}\}_{k\in\NN}$ with $\bbeta_0^{\circ}\sim\pi_{\alpha}$, we have
\begin{align}
    \||\bar\bbeta_k-\bar\bbeta_k^{\circ}|_s\|_q & = \frac{1}{k}\Big\|\Big|\sum_{i=1}^k(\bbeta_i-\bbeta_i^{\circ})\Big|_s\Big\|_q \nonumber \\
    & \le \frac{1}{k}\sum_{i=1}^k\||\bbeta_i-\bbeta_i^{\circ}|_s\|_q. 
\end{align}
For each $1\le i\le k$, it follows from the geometric-moment contraction in Theorem~\ref{thm_gmc} that
\begin{equation}
    \||\bbeta_i-\bbeta_i^{\circ}|_s\|_q \le  r_{\alpha,s,q}^i\||\bbeta_0-\bbeta_0^{\circ}|_s\|_q.
\end{equation}
Recall that $r_{\alpha,s,q}=1-2\mu\alpha +(\max\{q,s\}-1)L_{s,q}^2\alpha^2<1$ in~\eqref{eq_gmc}. Therefore,
\begin{align}
    \||\bar\bbeta_k-\bar\bbeta_k^{\circ}|_s\|_q \le \frac{1}{k} \cdot \frac{r_{\alpha,s,q}(1-r_{\alpha,s,q}^k)}{1-r_{\alpha,s,q}}\||\bbeta_0-\bbeta_0^{\circ}|_s\|_q \le \frac{1}{k}\cdot \frac{1}{1-r_{\alpha,s,q}}\||\bbeta_0-\bbeta_0^{\circ}|_s\|_q.
\end{align}
The desired result is achieved.
\end{proof}

\begin{proof}[Proof of Theorem~\ref{thm_bias}]
Without loss of generality, assume $\bbeta^*=0$. We use the notation~\eqref{eq_partials_G} for the derivatives of $G.$ Notice that
\begin{equation}
    \nabla G(\bbeta^*)=\nabla G(0)=0.
\end{equation}
Consider the stationary SGD recursion
\begin{align*}
    \bbeta_k^{\circ} = \bbeta_{k-1}^{\circ} - \alpha\nabla g(\bbeta_{k-1}^{\circ},\bm{\xi}_k),\quad k\ge1.
\end{align*}
By taking the expectation on the both sides, we obtain, for all $k\ge1$,
\begin{align}
    \EE [\nabla G(\bbeta_{k-1}^{\circ})] = 0.
\end{align}

Throughout the rest of the proof, we omit the iteration index $k$ and write $\bbeta=\bbeta_{k-1}^{\circ}$ when no confusion is caused. For notational convenience, write $\bbeta=(\beta_1,\ldots,\beta_d)^{\top}$.

A first-order Taylor expansion on $\nabla G(\bbeta)$ at $\bbeta^*=0$ gives
\begin{align}
    \label{eq_bias_taylor1}
    0 = \EE[\nabla G(\bbeta)] = \nabla G(0) + \nabla^2 G(0)\EE[\bbeta] + \R(\bbeta),
\end{align}
where $\nabla^2 G(0)$ is the $d\times d$ Jacobian matrix with entries defined by
\begin{align}
    [\nabla^2 G(0)]_{i,j} = \frac{\partial^2 }{\partial \beta_i\partial \beta_j} G(\bbeta)\Big|_{\bbeta=0}, \quad 1\le i, j\le d,
\end{align}
and $\R(\bbeta)$ is the $d$-dimensional remainder defined as
\begin{align}
    \R(\bbeta) & = \int_0^1\EE\big([\nabla^2 G(t\bbeta) - \nabla^2 G(0)]\bbeta\big) \,dt. 
\end{align}
The $i$-th entry of $\R(\bbeta)$ can be rewritten into
\begin{align}
    \R_i(\bbeta) & = \int_0^1(1-t)\EE\big(\bbeta^{\top} \nabla^3 G_i(t\bbeta) \bbeta\big) \,dt,
\end{align}
where $\nabla^3 G_i(\bbeta)$, $1\le i\le d$, is a $d\times d$ matrix whose entries are
\begin{align}
    [\nabla^3 G_i(\bbeta)]_{l,r} = \frac{\partial^3}{\partial\beta_i\partial\beta_l\partial\beta_r}G(\bbeta), \quad 1\le l,r\le d.
\end{align}

Since $\nabla G(0)=0$ and $\nabla^2 G(0)$ is invertible given that $\lambda_{\min}[\nabla^2 G(0)]>0$, it follows from equation~\eqref{eq_bias_taylor1} that
\begin{align}
    \EE[\bbeta] = -[\nabla^2 G(0)]^{-1}\EE[\R(\bbeta)].
\end{align}
We only need to bound $|\EE[\R(\bbeta)]|_s$ using Theorem~\ref{thm_sgd_moment}, that is $\EE[|\bbeta_k^{\circ}-\bbeta^*|_s]^2=O(\max\{q,s\}\alpha)$ for all $k\ge1$.

Let $\bm{v}=\bbeta/|\bbeta|_s$. For each $i=1,\ldots,d$,
\begin{align}
    \EE[\R_i(\alpha)] & = \int_0^1(1-t)\EE[\bbeta^{\top} \nabla^3 G_i(t\bbeta) \bbeta] \,dt \nonumber \\
    & = \int_0^1(1-t)\EE[|\bbeta|_s^2\bm{v}^{\top} \nabla^3 G_i(t\bbeta) \bm{v}] \,dt.
\end{align}
By Hölder's inequality, for $1/p+1/q=1$,
\begin{align}
    \EE[|\bbeta|_s^2\bm{v}^{\top} \nabla^3 G_i(t\bbeta) \bm{v}] & \le  (\EE[|\bbeta|_s^{2q}])^{1/q}\cdot (\EE(\bm{v}^{\top} \nabla^3 G_i(t\bbeta) \bm{v})^p)^{1/p}.
\end{align}
Again by Hölder's inequality,
\begin{align}
    \EE[(\bm{v}^{\top} \nabla^3 G_i(t\bbeta) \bm{v})^p] \le d^{p(1-\frac{2}{s})}\sup_{|\bm{v}|_s=1}\EE|\nabla^3 G_i(t\bbeta) \bm{v}|_s^p.
\end{align}
Therefore, by Theorem~\ref{thm_sgd_moment} and Lemma~\ref{lemma_Ls_L1_matrix_norm},
\begin{equation}
    \EE[\bbeta] \lesssim M_{s,q}^2\max\{q,s\}\alpha d^{\frac{q}{q-1}\cdot(1-\frac{2}{s})}\max_{1\le i\le d}\|\nabla^3 G_i(\bbeta)\|_{\infty},
\end{equation}
where the matrix norm
\begin{equation}
    \|\nabla^3 G_i(\bbeta)\|_{\infty}:=\max_{1\le j_1\le d}\sum_{j_2=1}^d\Big|\big(\nabla^3 G_i(\bbeta)\big)_{1\le j_1,j_2\le d}\Big|.
\end{equation}
Finally, given the uniform bound $\max_{1\le i\le d}\|\nabla^3 G_i(\bbeta)\|_{\infty}<\infty$,
\begin{equation}
    \EE[\bbeta] = O\Big(M_{s,q}^2\max\{q,s\}\alpha d^{\frac{q}{q-1}\cdot(1-\frac{2}{s})}\Big),
\end{equation}
which finishes the proof.
\end{proof}

\subsection{Proofs for Section~\ref{sec_concentration}}

\begin{proof}[Proof of Theorem~\ref{Theorem_Fuk_Nagaev}]
    By Theorem~\ref{thm_asgd_stationary}, we have $\||\bar{\bbeta}_{k} - \bar{\bbeta}_{k}^{\circ}|_{s}\|_{q} \lesssim 1/(k\alpha) \||\bbeta_{0} - \bbeta_{0}^{\circ}|_{s}\|_q$ and consequently, it follows that  
    \begin{align}
    \label{eq_Initial}
        \mathbb{P} (|\bar{\bbeta}_{k} - \bar{\bbeta}_{k}^{\circ}|_{s} > z) \lesssim \frac{\||\bbeta_{0} - \bbeta_{0}^{\circ}|_{s}\|_q^{q}}{(k \alpha z)^{q}}, \quad z > 0. 
    \end{align}
    Then it suffices to upper bound $\mathbb{P} (|\bar{\bbeta}_{k}^{\circ} - \bbeta^{*}|_{s} > z)$. To this end, we first bound the dependence adjusted norm (Section 2 in~\citet{zhang_gaussian_2017}) for $\{\bbeta_{k}^{\circ}\}_{k \geq 1}$. By Theorem~\ref{thm_gmc}, elementary calculations yield 
    \begin{align*}
        \||\bbeta_{k}^{\circ} - \mathbb{E} [\bbeta_{k}^{\circ}]|_{s}\|_{q, 1/2 - 1/q} = O\left(\frac{M_{s, q}}{\alpha^{1/2 - 1/q}}\right). 
    \end{align*}    
    Consequently, by Theorem 6.2 in~\citet{zhang_gaussian_2017} and Theorem~\ref{thm_bias}, we have 
    \begin{align*}
        \mathbb{P} (|\bar{\bbeta}_{k}^{\circ} - \bbeta^{*}|_{s} > z) \lesssim \frac{(\log d)^{3q/2} (\log k)^{1 + 2 q} M_{s, q}^{q}}{z^{q} k^{q - 1} \alpha^{q/2 - 1}} + \exp\left(- \frac{C k z^{2} \alpha^{1 - 2/q}}{M_{s, q}^{2} \log d}\right).  
    \end{align*}
    Combining this with~\eqref{eq_Initial} completes the proof. 
\end{proof}

\begin{theorem}[Theorem 3.1 in \citep{mies_sequential_2023}]
\label{thm.Mies}
Let $(\epsilon_i)_{i\in\ZZ}$ be i.i.d.\ random variables and $\bm{\epsilon}_k=(\epsilon_k,\epsilon_{k-1},\ldots)$. Assume $X_k=G_k(\bm{\epsilon}_k)\in\RR^d$ with $\EE[X_k]=0$ for some measurable function $G_k$. For any $k$, denote $\bm{\tilde\epsilon}_{k,j}=(\epsilon_k,\ldots,\epsilon_{j+1},\tilde\epsilon_j,\epsilon_{j-1},\ldots)$ with $\tilde\epsilon_j$ an i.i.d.\ copy of $\epsilon_j$. Assume there exist $\Theta>0$ and $q>2$, such that for all $k$,
\begin{align}
    (\EE|G_k(\bm{\epsilon}_k) - G_k(\bm{\tilde\epsilon}_{k,k-j})|_2^q)^{1/q} \le \frac{\Theta}{(j\vee1)^{3}}, \ \  \text{for all} \ j\ge0, \ \ \text{and \ \ } (\EE|G_k(\bm{\epsilon}_0)|_2^q)^{1/q} \le \Theta. 
    \label{eq.ciefe23e}
\end{align}
Additionally, assume that for some $\Gamma\ge1$,
\begin{align}
    \sum_{k=2}^n(\EE|G_k(\bm{\epsilon}_0) - G_{k-1}(\bm{\epsilon}_0)|_2^2)^{1/2} \le \Gamma\cdot\Theta.
    \label{eq.ciefe23e2}
\end{align}
If $d\le cn$ for some $c>0$, then on a potentially different probability space, there exist random vectors $(X_k')_{k=1}^n=^{\D}(X_k)_{k=1}^n$ and independent, mean zero, Gaussian random vectors 
\[Y_k^*\sim\N\Big(0,\sum_{h=-\infty}^{\infty}\mathrm{Cov}\big(G_k(\bm{\epsilon}_0),G_k(\bm{\epsilon}_h)\big)\Big)\] such that
\begin{align*}
    \Bigg(\EE\max_{m\le n}\Big|\frac{1}{\sqrt{n}}\sum_{k=1}^m(X_k'-Y_k^*)\Big|_2^2\Bigg)^{1/2} \le C\Theta\Gamma^{\frac{1}{4}}\sqrt{\log(n)}\Big(\frac{d}{n}\Big)^{\frac{q-2}{6q-4}},
\end{align*}
for some constant $C$ depending on $(q,c)$.
\end{theorem}

Instead of univariate $\epsilon_i$, we apply Theorem 3.1 with vector-valued i.i.d.\ inputs $\bm{\xi}_i$. The theorem still applies as the proof  depends only on the i.i.d.\ random elements and their $L^q$ bounds but not on the dimension of $\bm{\xi}_i.$

\begin{proof}[Proof of Theorem~\ref{thm_GA}]
To prove the Gaussian approximation we will apply Theorem \ref{thm.Mies} (Theorem 3.1 in \citet{mies_sequential_2023}) with $G_k\equiv G = h_{\alpha}$ defined in~\eqref{eq_beta_stationary_repre} since $\bbeta_k^{\circ}$ is stationary. We now verify the conditions \eqref{eq.ciefe23e} and \eqref{eq.ciefe23e2}. 

Recall the functional dependence measure $\||\bbeta_k^{\circ} - \bbeta_{k,\{l\}}^{\circ}|_s\|_q$ introduced in Section~\ref{subsubsec_functional_dep}.Throughout the proof, the $q$-th moment of the Euclidean norm is denoted by $$\|\cdot\|_q:=\big\||\cdot|_2\big\|_q.$$ Set
\begin{align}
    \rho_{\alpha,q}^2 := 1-2\alpha\mu+7(\max\{q,2\}-1)\alpha^2L_{2,q}^2, \quad C_{\alpha,q} := 2\alpha \sqrt{cL_{2,q}^2\max\{q,2\}\alpha + 1}
\end{align}
for some constant $c>0$. If $c$ is chosen sufficiently large, then, by Lemma~\ref{lemma_functional_dep} and Theorem~\ref{thm_sgd_moment}, for all $k\ge1$ and $l\le k$, we have
\begin{align*}
 \|\bbeta_k^{\circ}- \bbeta_{k,\{l\}}^{\circ}\|_q^2 
    & \le 4\alpha^2\rho_{\alpha,q}^{2(k-l)}\Big(L_{2,q}^2\|\bbeta_{l-1}^{\circ}-\bbeta^*\|_q^2 + M_{2,q}^2\Big) \nonumber \\
    & \le 4\alpha^2\rho_{\alpha,q}^{2(k-l)}M_{2,q}^2\Big(cL_{2,q}^2\max\{q,2\}\alpha + 1\Big) \\
    &=C_{\alpha,q}^2\rho_{\alpha,q}^{2(k-l)}M_{2,q}^2.
\end{align*}
For $\alpha\in(0,\alpha_{s_d,q})$, it follows that $\rho_{\alpha,q}<1$. Let $l=k-j$. Then, for a sufficiently large constant $C_{\alpha, q}',$ we have
\begin{align}
    \label{eq_func_dep_gmc}
    \| \bbeta_k^{\circ}- \bbeta_{k,\{k-j\}}^{\circ}\|_q & \le C_{\alpha,q}M_{2,q}\rho_{\alpha,q}^j \le C_{\alpha,q}' M_{2,q}(j+1)^{-3}.
\end{align}
Therefore, the condition \eqref{eq.ciefe23e} holds with $\Theta= C_{\alpha,q}' M_{2,q}$. This verifies the first part of condition~\ref{eq.ciefe23e}. For the second part of condition~\ref{eq.ciefe23e}, by Assumption~\ref{asm_Ls_lip} and Theorem~\ref{thm_sgd_moment}, for some constant $C_{\alpha,q}''>0$,
\begin{align}
    \|h_{\alpha}(\bm{\xi}_0,\bm{\xi}_{-1},\ldots)\|_q = \|\bbeta_0^{\circ}\|_q \le \|\bbeta_0^{\circ}-\bbeta^*\|_q + |\bbeta^*|_2 \le C_{\alpha,q}''M_{2,q} <\infty.
\end{align} 
Moreover, since $\bbeta_k^{\circ}$ is stationary, $G_k=G_{k-1}=h_{\alpha}$ and the left hand side of~\eqref{eq.ciefe23e2} is zero. Thus, condition~\eqref{eq.ciefe23e2} is trivially satisfied with $\Gamma=1$.

Finally, we show that the long-run covariance matrix $\Xi=\sum_{k=-\infty}^{\infty}\mathrm{Cov}(\bbeta_0^{\circ},\bbeta_k^{\circ})$ is well defined in the sense that the spectral norm $\|\Xi\|_s$ is finite. Following~\eqref{eq_functional_dep_def}, denote 
\begin{equation}
    \bbeta_{k,\{\le l\}}^{\circ} := \bbeta_{k,\{\ldots,l-1,l\}}^{\circ} = h_{\alpha}(\bm{\xi}_k,\ldots,\bm{\xi}_{l+1},\bm{\xi}_l',\bm{\xi}_{l-1}',\ldots)=h_{\alpha}(\F_{k,\{\ldots,l-1,l\}}), \quad l\le k.
\end{equation}
Since $\bbeta_{k,\{\le0\}}^{\circ}$ is independent of $\bbeta_0^{\circ}$, we have
\begin{align}
    \mathrm{Cov}(\bbeta_0^{\circ},\bbeta_k^{\circ}) & = \EE[\bbeta_0^{\circ}\bbeta_k^{\circ\top}] - \EE[\bbeta_0^{\circ}]\EE[\bbeta_k^{\circ\top}] \nonumber \\
    & = \EE[\bbeta_0^{\circ}\bbeta_k^{\circ\top}] - \EE[\bbeta_0^{\circ}]\EE[\bbeta_{k,\{\le0\}}^{\circ\top}] \nonumber \\
    & = \EE\big[\bbeta_0^{\circ} (\bbeta_k^{\circ} - \bbeta_{k,\{\le 0\}}^{\circ})^{\top}\big].
\end{align}
We can rewrite the difference as a telescoping sum,
\begin{align}
    \bbeta_k^{\circ} - \bbeta_{k,\{\le 0\}}^{\circ} = \sum_{l=0}^{\infty}\big(\bbeta_{k,\{\le -l+1\}}^{\circ} - \bbeta_{k,\{\le -l\}}^{\circ}\big).
\end{align}
By stationarity and~\eqref{eq_func_dep_gmc}, it follows that
\begin{align}
    \big\|\bbeta_{k,\{\le -l+1\}}^{\circ} - \bbeta_{k,\{\le -l\}}^{\circ}\big\|_2 = \big\|\bbeta_{k,\{-l+1\}}^{\circ} - \bbeta_{k,\{-l\}}^{\circ}\big\|_2 \le C_{\alpha,2}M_{2,2}\rho_{2,2}^{k+l+1}.
\end{align}
For the spectral norm,
\begin{align}
    \big\|\mathrm{Cov}(\bbeta_0^{\circ},\bbeta_k^{\circ})\big\|_s 
    & = \sup_{\bm{u},\bm{v}\in\RR^d,|\bm{u}|_2=|\bm{v}|_2=1}\EE\bm{v}^{\top}\bbeta_0^{\circ}(\bbeta_k^{\circ} - \bbeta_{k,\{\le 0\}}^{\circ})^{\top}\bm{u} \nonumber \\
    & \le \sup_{\bm{u},\bm{v}\in\RR^d,|\bm{u}|_2=|\bm{v}|_2=1}[\EE(\bm{v}^{\top}\bbeta_0^{\circ})^2]^{1/2}\big[\EE[(\bbeta_k^{\circ} - \bbeta_{k,\{\le 0\}}^{\circ})^{\top}\bm{u}]^2\big]^{1/2} \nonumber \\
    & \le \big\|\bbeta_0^{\circ}\big\|_2\big\|\bbeta_k^{\circ} - \bbeta_{k,\{\le 0\}}^{\circ}\big\|_2,
\end{align}
where the first inequality is by Cauchy-Schwarz and the last inequality uses  $(\bm{u}^{\top}\bbeta_0^{\circ})^2\le |\bbeta_0^{\circ}|^2$ with $|\bm{u}|_2=1$. This, along with $M_{2,2}<\infty$ (Assumption~\ref{asm_Ls_lip}) yields
\begin{align}
    \big\|\mathrm{Cov}(\bbeta_0^{\circ},\bbeta_k^{\circ})\big\|_s \le \|\bbeta_0^{\circ}\|_2\|\bbeta_k^{\circ} - \bbeta_{k,\{\le 0\}}^{\circ}\|_2 \le C_{\alpha}'\rho_{\alpha,2}^k,
\end{align}
for some constant $C_{\alpha}'>0$. As a direct consequence,
\begin{align}
    \|\Xi\|_s \le \big\|\EE[\bbeta_0^{\circ}\bbeta_0^{\circ\top}]\big\|_s + 2\sum_{k=1}^{\infty}C_{\alpha}'\rho_{\alpha,2}^k<\infty.
\end{align}
This completes the proof.

\end{proof}


\newpage

\setcounter{page}{1}
\section*{NeurIPS Paper Checklist}

\begin{enumerate}

\item {\bf Claims}
    \item[] Question: Do the main claims made in the abstract and introduction accurately reflect the paper's contributions and scope?
    \item[] Answer: \answerYes{} 
    \item[] Justification: 
    \textit{We clearly stressed our contributions and scope in the abstract and included a subsection in the introduction to list our key innovations.}

\item {\bf Limitations}
    \item[] Question: Does the paper discuss the limitations of the work performed by the authors?
    \item[] Answer: \answerYes{} 
    \item[] Justification: 
    \textit{We talked about the limitations in the last section, especially for the assumptions of strong convexity and smoothness.}

\item {\bf Theory assumptions and proofs}
    \item[] Question: For each theoretical result, does the paper provide the full set of assumptions and a complete (and correct) proof?
    \item[] Answer: \answerYes{} 
    \item[] Justification: 
    \textit{We provide rigorous proofs for all the theoretical results in the supplementary material.}

    \item {\bf Experimental result reproducibility}
    \item[] Question: Does the paper fully disclose all the information needed to reproduce the main experimental results of the paper to the extent that it affects the main claims and/or conclusions of the paper (regardless of whether the code and data are provided or not)?
    \item[] Answer: \answerNA{} 
    \item[] Justification: 
    \textit{This paper contributes to theoretical guarantees of high-dimensional SGD.}

\item {\bf Open access to data and code}
    \item[] Question: Does the paper provide open access to the data and code, with sufficient instructions to faithfully reproduce the main experimental results, as described in supplemental material?
    \item[] Answer: \answerNA{} 
    \item[] Justification: 
    \textit{This paper does not include experiments requiring code.}

\item {\bf Experimental setting/details}
    \item[] Question: Does the paper specify all the training and test details (e.g., data splits, hyperparameters, how they were chosen, type of optimizer, etc.) necessary to understand the results?
    \item[] Answer: \answerNA{} 
    \item[] Justification: 
    \textit{This paper does not include experiments.}

\item {\bf Experiment statistical significance}
    \item[] Question: Does the paper report error bars suitably and correctly defined or other appropriate information about the statistical significance of the experiments?
    \item[] Answer: \answerNA{} 
    \item[] Justification: 
    \textit{This paper does not include experiments.}

\item {\bf Experiments compute resources}
    \item[] Question: For each experiment, does the paper provide sufficient information on the computer resources (type of compute workers, memory, time of execution) needed to reproduce the experiments?
    \item[] Answer: \answerNA{} 
    \item[] Justification: 
    \textit{This paper does not include experiments.}
    
\item {\bf Code of ethics}
    \item[] Question: Does the research conducted in the paper conform, in every respect, with the NeurIPS Code of Ethics \url{https://neurips.cc/public/EthicsGuidelines}?
    \item[] Answer: \answerYes{} 
    \item[] Justification: 
    \textit{We followed the NeurIPS Code of Ethics as instructed.}

\item {\bf Broader impacts}
    \item[] Question: Does the paper discuss both potential positive societal impacts and negative societal impacts of the work performed?
    \item[] Answer: \answerYes{} 
    \item[] Justification: 
    \textit{This paper discusses potential positive societal impacts, particularly through advancing the theoretical understanding of modern machine learning, which can inform the development of more robust and efficient algorithms. As the work is purely theoretical and does not propose or evaluate any deployable systems, we do not anticipate any direct negative societal impacts.}
    
\item {\bf Safeguards}
    \item[] Question: Does the paper describe safeguards that have been put in place for responsible release of data or models that have a high risk for misuse (e.g., pretrained language models, image generators, or scraped datasets)?
    \item[] Answer: \answerNA{} 
    \item[] Justification: 
    \textit{This paper poses no such risks.}

\item {\bf Licenses for existing assets}
    \item[] Question: Are the creators or original owners of assets (e.g., code, data, models), used in the paper, properly credited and are the license and terms of use explicitly mentioned and properly respected?
    \item[] Answer: \answerNA{} 
    \item[] Justification: 
    \textit{This paper does not use existing assets.}

\item {\bf New assets}
    \item[] Question: Are new assets introduced in the paper well documented and is the documentation provided alongside the assets?
    \item[] Answer: \answerNA{} 
    \item[] Justification: 
    \textit{This paper does not release new assets.}

\item {\bf Crowdsourcing and research with human subjects}
    \item[] Question: For crowdsourcing experiments and research with human subjects, does the paper include the full text of instructions given to participants and screenshots, if applicable, as well as details about compensation (if any)? 
    \item[] Answer: \answerNA{} 
    \item[] Justification: 
    \textit{This paper does not involve crowdsourcing nor research with human subjects.}

\item {\bf Institutional review board (IRB) approvals or equivalent for research with human subjects}
    \item[] Question: Does the paper describe potential risks incurred by study participants, whether such risks were disclosed to the subjects, and whether Institutional Review Board (IRB) approvals (or an equivalent approval/review based on the requirements of your country or institution) were obtained?
    \item[] Answer: \answerNA{} 
    \item[] 
    \textit{This paper does not involve crowdsourcing nor research with human subjects.}

\item {\bf Declaration of LLM usage}
    \item[] Question: Does the paper describe the usage of LLMs if it is an important, original, or non-standard component of the core methods in this research? Note that if the LLM is used only for writing, editing, or formatting purposes and does not impact the core methodology, scientific rigorousness, or originality of the research, declaration is not required.
    \item[] Answer: \answerNA{} 
    \item[] Justification: 
    \textit{The core method development in this research does not involve LLMs as any important, original, or non-standard components.}

\end{enumerate}

\end{document}